\documentclass[english,notitlepage]{elsarticle}
\usepackage[T1]{fontenc}
\usepackage[latin9]{inputenc}
\usepackage{array}
\usepackage{float}
\usepackage{amsthm}
\usepackage{amsmath}
\usepackage{amssymb}

\makeatletter

\newcommand{\noun}[1]{\textsc{#1}}
\providecommand{\tabularnewline}{\\}
\floatstyle{ruled}
\newfloat{algorithm}{tbp}{loa}
\providecommand{\algorithmname}{Algorithm}
\floatname{algorithm}{\protect\algorithmname}

\usepackage{enumitem}		
  \theoremstyle{definition}
  \newtheorem{defn}{\protect\definitionname}
 \theoremstyle{definition}
  \newtheorem{example}{\protect\examplename}
  \theoremstyle{definition}
  \newtheorem{problem}{\protect\problemname}
  \theoremstyle{plain}
  \newtheorem{assumption}{\protect\assumptionname}
\theoremstyle{plain}
\newtheorem{thm}{\protect\theoremname}
  \theoremstyle{plain}
  \newtheorem{lem}{\protect\lemmaname}
  \theoremstyle{plain}
  \newtheorem{cor}{\protect\corollaryname}
  \theoremstyle{plain}
  \newtheorem{prop}{\protect\propositionname}

\usepackage{todonotes}
\usepackage{algorithmicx}
\usepackage{algpseudocode}
\usepackage{thmtools}
\usepackage{thm-restate}
\usepackage{bbold}
\usepackage{tikz}
\usetikzlibrary{backgrounds,calc,trees}
\algnewcommand{\LineComment}[1]{\State \(\triangleright\) #1}
\DeclareMathOperator*{\argmax}{arg\,max}
\newtheorem{counter}{Counterexample}
\newcommand{\conf}[1]{\mathbf{#1}}
\newcommand{\EGPName}{{\sc Extend-To-Global-Projection}}
\newcommand{\ESName}{{\sc Extend-To-Subtree}}
\newcommand{\SESName}{{\sc Single-Extend-To-Subtree}}
\renewcommand*{\appendixname}{}

\@ifundefined{showcaptionsetup}{}{%
 \PassOptionsToPackage{caption=false}{subfig}}
\usepackage{subfig}
\makeatother

\usepackage{babel}
  \providecommand{\assumptionname}{Assumption}
  \providecommand{\definitionname}{Definition}
  \providecommand{\examplename}{Example}
  \providecommand{\lemmaname}{Lemma}
  \providecommand{\problemname}{Problem}
  \providecommand{\propositionname}{Proposition}
\providecommand{\corollaryname}{Corollary}
\providecommand{\theoremname}{Theorem}

\begin{document}

\begin{frontmatter}{}

\title{Sufficient and necessary conditions for Dynamic Programming in Valuation-Based
Systems. }

\author[iiia]{Jordi Roca-Lacostena}

\ead{jroca@iiia.csic.es}

\author[iiia]{Jesus Cerquides}

\ead{cerquide@iiia.csic.es}

\author[lucerne]{Marc Pouly}

\ead{marc.pouly@hslu.ch}

\fntext[iiia]{IIIA-CSIC, Campus UAB, 08193 Cerdanyola, Spain}

\fntext[lucerne]{Lucerne University of Applied Sciences and Arts, Technikumstrasse
21, 6048 Horw, Switzerland}
\begin{abstract}
Valuation algebras abstract a large number of formalisms for automated
reasoning and enable the definition of generic inference procedures.
Many of these formalisms provide some notion of solution. Typical
examples are satisfying assignments in constraint systems, models
in logics or solutions to linear equation systems. 

Many widely used dynamic programming algorithms for optimization problems
rely on low treewidth decompositions and can be understood as particular
cases of a single algorithmic scheme for finding solutions in a valuation
algebra. The most encompassing description of this algorithmic scheme
to date has been proposed by Pouly and Kohlas together with sufficient
conditions for its correctness. Unfortunately, the formalization relies
on a theorem for which we provide counterexamples. In spite of that,
the mainline of Pouly and Kohlas' theory is correct, although some
of the necessary conditions have to be revised. In this paper we analyze
the impact that the counter-examples have on the theory, and rebuild
the theory providing correct sufficient conditions for the algorithms.
Furthermore, we also provide necessary conditions for the algorithms,
allowing for a sharper characterization of when the algorithmic scheme
can be applied.
\end{abstract}

\end{frontmatter}{}

\section{Introduction}

Solving optimization problems is an important and well-studied task
in computer science. There are many optimization problems whose solution
can be expressed as an assignment of values to a set of variables.
Usually, the larger the number of variables involved in the problem,
the more complex it is to find a solution. A particular approach to
tackle problems whose solution involves a large number of variables
is known as dynamic programming \cite{Bellman1957} and can be found
in almost every handbook about algorithms and programming techniques
\cite{Cormen2009,Skiena2008}. 

The initial works of Bellman and Dreyfus \cite{Bellman1957,Bellman1962}
studied the problem from a decision making perspective and used the
term optimal policy instead of solution and decision instead of variable.
They advocated solving the problem by performing a sequence of steps,
which they associated with an artificial time-like property, hence
the name dynamic. At each step, the values for some of the variables
were determined, based on the values determined in the previous steps.
These works establish the basis of serial dynamic programming. In
order to understand when such a technique could be applied, Bellman
enunciated the \emph{Principle of Optimality:}
\begin{quotation}
An optimal policy has the property that whatever the initial state
and initial decision are, the remaining decisions must constitute
an optimal policy with regard to the state resulting from the first
decision.
\end{quotation}
Different formalizations of the principle have been proposed. Karp
and Held \cite{Karp1967} concentrate on the sequential nature of
dynamic programming. Non-serial dynamic programming is introduced
later, among others, by Bertelè and Brioschi \cite{Bertele1972,Esogbue1974}.
Helman \cite{Helman1989} formalizes a wider view of dynamic programming
based on the idea of computationally feasible dominance relations.
This formalization is later reformulated in a categorical setting
by Bird and de Moor \cite{Bird1997} and successfully translated into
a generic program\footnote{Here we refer to the \emph{generic programming} idea of Dehnert and
Stepanov \cite{Dehnert2000,Stepanov2014} of trying to provide algorithms
that work in the most general setting without loss of efficiency.} \cite{DeMoor1995,DeMoor1999}.  

More recently, Lew and Mauch \cite{Lew2006} proposed a formalization
that takes as central object the dynamic programming functional equation,
which can be automatically translated into efficient code. Also, Sniedovich
\cite{Sniedovich2010,Sniedovich2006} explored the fundations of dynamic
programming presenting a ``recipe'' and formally defining a \emph{decomposition
scheme} as the key concept for dynamic programming. Hovewer, in each
of these later works, dynamic programming is presented as an algorithm
that can be applied to optimize functions taking values in the real
numbers.

Some of the later research \cite{Buresh-Oppenheim2011,Jukna2014}
is concerned with finding more constrained models for dynamic programming,
which enable the finding of limitations for dynamic programming solutions.

In a parallel and more algebraic path of research lies the approach
taken by Mitten \cite{Mitten1964a} and further generalized by Shenoy
\cite{Shenoy1996}, for functions taking values in any ordered set
$\Delta$. Shenoy introduces a set of axioms that later on will be
known as valuation algebras. In those terms, Shenoy is the first to
connect the concept of solution with the projection operation of the
valuation algebra. 

In a further generalization effort, Pouly and Kohlas \cite{Pouly2011a,Pouly2011c}
drop the assumption that valuations are functions that map tuples
into a value set $\Delta.$ They introduce three different algorithms,
that we have named \EGPName, \ESName\  and \noun{\SESName} and
provide sufficient conditions for their correctness. Pouly and Kohlas'
algorithms are more general than their predecessors in the literature.
This increased generality comes at no computational cost, since when
applied in the previously covered scenarios, their particularization
coincides exactly with the previously proposed algorithm. Furthermore,
by dropping the assumption that valuations are functions, their algorithms
can be applied to previously uncovered cases such as the solution
of linear equation systems or the algebraic path problem\cite{Zimmermann1981}.

Against this background, in this paper we establish by means of counterexamples
that, unfortunately, one of the fundamental theorems in Pouly and
Kohlas' theory is incorrect. Since the theorem is used in the proofs
of several other results in their work, uncertainty spreads over the
truth of these now potentially falsable results. In the paper we analyze
the impact on the theory and clarify which statements were true but
incorrectly proven and which of them were false. For the true ones,
we provide a correct proof whilst for the false ones we identify the
additional conditions required for their correctness. 

The contribution of the paper is not limited to correcting Pouly and
Kohlas' theory. We do introduce two new concepts: projective completability
and piecewise completability. We show that projective completability
is a sufficient condition for the \EGPName\  algorithm, whereas piecewise
completability is a sufficient condition for the \ESName\  algorithm.
Furthermore, we do also show that they are a necessary condition.
To the best of our knowledge, this is the first time in which necessary
conditions for dynamic programming algorithms on valuation-based systems
are identified. 

A particularly relevant subfamily of valuation algebras, known as
\emph{semiring induced valuation algebras }\cite{Kohlas2008}, underlie
the foundation of many important artificial intelligence formalisms
such as constraint systems, probability potentials for Bayesian networks
or Spohn potentials. Many optimization problems can be formalized
by means of the valuation algebra induced by a selective conmutative
semiring. We revise the sufficient conditions (defined in terms of
properties of the semiring) proposed by Pouly and Kohlas \cite{Pouly2011a,Pouly2011c},
and provide correct sufficient conditions for each of the algorihtms.
Furthermore, where possible we also provide necessary conditions.

The paper is structured as follows. In section \ref{sec:Background}
we review valuation algebras, covering join trees and the basic algorithms
for assessing one projection \noun{(Collect}) and several projections
(\noun{Collect+Distribute}) of a factorized valuation. After that,
in section \ref{sec:GenericSolutions} we present the solution finding
problem, the abstract problem underlying optimization problems, and
we show by means of counterexamples that one of the results in Pouly
and Kohlas' work is not correct. Later, in section \ref{sec:Impact}
we analyze why disproving the result has a deep impact on the theory.
As a consequence, in section \ref{sec:NewConditions} we identify
new sufficient conditions for the algorithms. Furthermore, we prove
that these conditions are also necessary. Since our conditions are
weaker, we can use them to provide new proofs for the results in Pouly
and Kohlas' theory affected by the counterexamples. Then, in section
\ref{sec:Optimization}, we study the specific case of semiring induced
valuation algebras and provide sufficient and necessary conditions
there in terms of properties of the semiring. Finally, we conclude
in section \ref{sec:Conclusions}.
\begin{quotation}
\end{quotation}

\section{Background\label{sec:Background}}

 In this section we start by defining valuation algebras. Later on,
we introduce the problem of assessing the projection of a factorized
valuation and review the \noun{Collect} algorithm to solve that problem.
Finally we review the algorithm used to assess multiple projections
of a factorized valuation.

\subsection{Valuation algebras}

The basic elements of a valuation algebra are so-called \emph{valuations},
that we subsequently denote by lower-case Greek letters such as $\phi$
or $\psi.$ Let $\Phi$ be a set of valuations and $U=\{u_{1},u_{2},\dots,u_{|U|}\}$
be a finite set of variables. A valuation algebra $(\Phi,U)$ has
three operations:
\begin{enumerate}
\item \emph{Labeling: $\Phi\rightarrow\mathcal{P}(U);\phi\mapsto d(\phi),$}
\item \emph{Combination: $\Phi\times\Phi\rightarrow\Phi;(\phi,\psi)\mapsto\phi\times\psi,$}
\item \emph{Projection: $\Phi\times\mathcal{P}(U)\rightarrow\Phi;(\phi,X)\mapsto\phi^{\downarrow X}$
}for \emph{$X\subseteq d(\phi).$}
\end{enumerate}
satisfying the following axioms:
\begin{description}
\item [{A1}] \emph{Commutative semigroup}: $\Phi$ is associative and commutative
under $\times.$
\item [{A2}] \emph{Labeling: }For $\psi,\phi\in\Phi,$ $d(\phi\times\psi)=d(\phi)\cup d(\psi).$
\item [{A3}] \emph{Projection: }For $\phi\in\Phi,$ and $X\subseteq d(\phi),$
$d(\phi^{\downarrow X})=X.$
\item [{A4}] \emph{Transitivity: }For $\phi\in\Phi$ and $X\subseteq Y\subseteq d(\phi),$
$(\phi^{\downarrow Y})^{\downarrow X}=\phi^{\downarrow X}.$
\item [{A5}] \emph{Combination: }For $\phi,\psi\in\Phi$ with $d(\phi)=X$,
$d(\psi)=Y$, and $Z\in\mathcal{P}(U)$ such that $X\subseteq Z\subseteq X\cup Y,$
$(\psi\times\phi)^{\downarrow Z}=\phi\times\psi^{\downarrow(Z\cap Y)}.$
\item [{A6}] \emph{Domain: }For $\phi\in\Phi$ with $d(\phi)=X,$ $\phi^{\downarrow X}=\phi.$
\end{description}
We say that a valuation $e\in\Phi$ is an identity valuation provided
that $d(e)=\emptyset$ and $\phi\times e=\phi$ for each $\phi\in\Phi.$
As proven in \cite{Pouly2012}, any valuation algebra that does not
have and identity valuation can easily be extended to have one. In
the following and without loss of generality we assume that our valuation
algebra has an identity valuation $e.$ Let $\Gamma=\{\gamma_{1},\ldots,\gamma_{n}\}$
be a set of valuations. We define $\prod_{\gamma\in\Gamma}\gamma$
as $e\times\gamma_{1}\times\cdots\times\gamma_{n}.$

\begin{defn}
Let $U$ be a finite set of variables and let $D_{i}$ denote the
domain of variable $u_{i},$ i.e. the set of its possible values.
Define further $D=\bigcup_{i=1}^{|U|}D_{i}$. A tuple $\mathbf{x}$
with domain $X\in\mathcal{P}(U)$ is a map $\mathbf{x}:X\rightarrow D$
such that $\mathbf{x}(u_{i})\in D_{i}$ for all $u_{i}\in X.$ Let
$\Omega_{X}$ denote the set of all tuples with domain $X$ if $X\neq\emptyset$
and set $\Omega_{\emptyset}=\{\diamond\}$ where $\diamond$ is introduced
for convenience and can be understood as the empty tuple. We denote
the set of all tuples as $\Omega=\bigcup_{X\in\mathcal{P}(U)}\Omega_{X}.$
A pair $\langle U,\Omega\rangle$ is known as a \emph{variable system}.
\end{defn}
Three basic operations are defined on tuples:
\begin{enumerate}
\item Labeling: $\Omega\rightarrow\mathcal{P}(U)$;$\mathbf{x}\mapsto d(\mathbf{x})$
such that $d(\mathbf{x})=X$ if and only if $\mathbf{x}\in\Omega_{X}.$
\item Projection: $\Omega\times\mathcal{P}(U)\rightarrow\Omega$;$(\mathbf{x},Y)\mapsto\mathbf{x}^{\downarrow Y}$,
defined when $Y\subseteq d(\mathbf{x})$, where $\mathbf{x}^{\downarrow Y}$
is a tuple with domain $Y$ defined as $\mathbf{x}^{\downarrow Y}(u_{i})=\mathbf{x}(u_{i})$
for any $u_{i}\in Y$ if $Y\neq\emptyset$ and $\mathbf{x}^{\downarrow\emptyset}=\diamond.$ 
\item Concatenation: $\Omega\times\Omega\rightarrow\Omega$;$(\mathbf{x},\mathbf{y})\mapsto\langle\mathbf{x},\mathbf{y}\rangle,$
defined when $\mathbf{x}^{\downarrow d(\mathbf{x})\cap d(\mathbf{y})}=\mathbf{y}^{\downarrow d(\mathbf{x})\cap d(\mathbf{y})}$,
where $\mathbf{z}=\langle\mathbf{x},\mathbf{y}\rangle$ is a tuple
with domain $d(\mathbf{x})\cup d(\mathbf{y})$ such that $\mathbf{z}(u_{i})=\begin{cases}
\mathbf{x}(u_{i}) & \mbox{ if }u_{i}\in d(\mathbf{x})\\
\mathbf{y}(u_{i}) & \mbox{ otherwise.}
\end{cases}$ 
\end{enumerate}
Note that, although sharing the same name, the labeling and projection
operations on tuples are not connected to the equivalently named operations
defined on valuations. 

We illustrate the previous concepts with an example of valuation algebra.
\begin{example}
\label{ex:ValAlgBooleanFunctions}

Let $U$ be a finite set of binary variables (that is, for each $u_{i}\in U,$
$D_{i}=\{0,1\})$. The set of valuations $\Phi$ is composed of all
the functions $\phi:\Omega_{X}\rightarrow\{0,1\}$ , where $X\subseteq U$.
The labeling operation is defined by $d(\phi)=X$. The combination
of two valuations $\phi,\psi$ is the valuation \emph{$(\phi\times\psi)(\mathbf{x})=\phi(\mathbf{x}^{\downarrow d(\phi)})\cdot\psi(\mathbf{x}^{\downarrow d(\psi)}),$}
where $\cdot$ is the boolean product. The projection of a valuation
$\phi$ with $d(\phi)=X$ to a domain $Y\subseteq X$ is the valuation
\emph{$\phi^{\downarrow Y}(\mathbf{y})={\displaystyle \max{}_{\mathbf{z}\in\Omega_{X-Y}}\phi(\langle\mathbf{y},\mathbf{z}\rangle).}$
}As proven in \cite{Kohlas2003} this valuation algebra of indicator
functions satisfies axioms A1-A6.

In this paper we will be interested in valuation algebras with a variable
system. Some relevant examples are relational algebra, which is fundamental
to databases, the algebra of probability potentials, which underlies
many results in probabilistic graphical models and the more abstract
class of semiring induced valuation algebras \cite{Kohlas2008,Pouly2011a}.
\end{example}

\subsection{Assessing the projection of a factorized valuation}

A relevant problem in many valuation algebras is the problem of assessing
the projection of a factorized valuation. 
\begin{problem}
\label{pbm:Projection}Let $(\Phi,U)$ be a valuation algebra, $\phi_{1},\ldots,\phi_{n}$
be valuations in $\Phi$, and $X\subseteq d(\phi_{1})\cup\dots\cup d(\phi_{n})$.
Assess $\left(\phi_{1}\times\dots\times\phi_{n}\right)^{\downarrow X}.$
\end{problem}
Note that when our valuations are probability potentials, this is
the well studied problem of assessing the marginal of a factorized
distribution, also known as Markov Random Field. 

The \noun{Fusion} algorithm \cite{Shenoy1991} (a.k.a. variable elimination)
or the \noun{Collect} algorithm (a.k.a. junction tree or cluster tree
algorithm)\cite{Pouly2011a,Pouly2011c} can be used to assess projections
of factorized valuations. Since our results build on top of the \noun{Collect}
algorithm, we provide a more accurate description below.

A necessary condition to run the \noun{Collect} algorithm is organizing
the valuations $\phi_{1},\ldots,\phi_{n}$ into a covering join tree,
which we introduce after some basic definitions. 

An \emph{undirected graph} is a pair $(V,E),$ where $V$ is a set
of nodes and a $E\subseteq\{\{i,j\}|i\in V,j\in V\}$ is a set of
edges. The set of neighbors of a node $i$ is $ne(i)=\{j|\{i,j\}\in E\}.$
A \emph{tree }is a undirected connected graph without loops. A \emph{labeled
tree} is any tree $(V,E)$ together with a function $\lambda:V\rightarrow\mathcal{P}(U)$
that links each node with a single domain in $\mathcal{P}(U)$. A
\emph{join tree} is a labeled tree $\mathcal{\mathcal{T}}=(V,E,\lambda,U)$
such that for any $i,j\in V$ it holds that $\lambda(i)\cap\lambda(j)\subseteq\lambda(k)$
for all nodes $k$ on the path between $i$ and $j$. In that case,
we say that $\mathcal{T}$ satisfies the running intersection property.
For each edge $\{i,j\}\in E,$ we define the separator between $i$
and $j$ as $s_{ij}=\lambda(i)\cap\lambda(j).$

\begin{defn}
\label{def:CoveringJoinTree}Given a valuation $\phi$ that factorizes
as $\phi=\phi_{1}\times\dots\times\phi_{n},$ we say that a join tree
$\mathcal{\mathcal{T}}=(V,E,\lambda,U)$ is a \emph{covering join
tree} for this factorization if for all $\phi_{j}$ there is a node
$i\in V$ such that $d(\phi_{j})\subseteq\lambda(i)$ . In that case
it is always possible to define a \emph{valuation assignment}, that
is a function $a:\{1,\ldots n\}\rightarrow V$, such that for all
$j\in\{1,\ldots,n\},$ $d(\phi_{j})\subseteq\lambda(a(j))$, that
assigns each valuation to one and only one of the nodes in the tree.
Thus, given a node $j,$ $a^{-1}(j)$ stands for the set of valuations
which are assigned to node $j.$ For each node $i$ in the covering
join tree we define $\psi_{i}=\prod_{j\in a^{-1}(i)}\phi_{j}.$ Note
that $\phi$ factorizes as $\phi=\prod_{i\in V}\psi_{i}.$ 
\end{defn}
The complexity of each of the algorithms presented in the paper increases
with the cardinality of $\lambda(i).$ Thus we want our sets $\lambda(i)$
to be as small as possible. In this work we will make the assumption
that the covering join trees are minimally labelled.
\begin{assumption}
\label{assu:Minimal}The nodes in a covering join tree are minimally
labeled, that is for each $i\in V$, and for each $k\in ne(i)$ 
\begin{equation}
\lambda(i)=d(\psi_{i})\cup\bigcup_{j\in ne(i)-\{k\}}s_{ij}.\label{eq:MinimalLabeling}
\end{equation}
 
\end{assumption}
Intuitively, the assumption means that the scope of a node does not
contain unnecessary variables. Note that given a tree and a valuation
assignment $a$, there is an easy way\footnote{See appendix \ref{sec:MinimallyLabeledJoinTrees} for more details}
to assess a minimally labelled covering join tree. Since, the so assessed
tree leads to smaller costs for the algorithms, the assumption can
be considered to be without loss of generality from a practical point
of view and simplifies the proofs.

\begin{defn}
A \emph{rooted join tree} is a join tree where one of the nodes has
been designated as root. Let $i$ be a node in a rooted join tree
whose root is $r$. The \emph{parent} of a node $i$, $p_{i}$ is
the node directly connected to it on the path to the root. Every node
except the root has a unique parent. The \emph{separator} of $i,$
$s_{i}$ is defined as $s_{i}=\begin{cases}
\emptyset & \mbox{\ \ if }i=r\\
s_{ip_{i}} & \mbox{\ \ otherwise}
\end{cases}.$\\
We note $ch(i)$, the set containing the children of $i$ (those nodes
whose parent is $i$), $de(i)$ the set containing the descendants
of $i$ (those nodes that have $i$ in their path to the root), and
$nde(i)$ as the set containing those nodes of $\mathcal{T}$ which
are not descendants of $i,$ namely $nde(i)=V\setminus(de(i)\cup\{i\}).$

\end{defn}

\begin{defn}
Let $I=\langle i_{1},\ldots,i_{n}\rangle$ be an ordering of the nodes
of the rooted tree $\mathcal{T}.$ We say that $I$ is \emph{upward}
if every node appears after all of its children. We say that $I$
is \emph{downward} if every node appears before any of its children.
\end{defn}
Algorithm \ref{alg:Collect} provides a description of the \noun{Collect}
algorithm. It is based on sending messages upwards, through the edges
of the covering join tree, until the root node is reached. The message
$\mu_{i\rightarrow p_{i}}$ sent from node $i$ to its parent summarizes
the information in the subtree rooted at $i$ which is relevant to
its parent. The running intersection property guarantees that no information
is lost.

\begin{algorithm}
{\footnotesize{}\begin{algorithmic}[1]
\ForAll{nodes $i$  of $\mathcal T$}
	\State $\psi_i := \prod_{j\in a^{-1}(i)}\phi_j$
	\State $\psi'_i := \psi_i$
\EndFor
\ForAll{nodes $i$ of $\mathcal T$ except the root in an upward order }
	\State $\mu_{i \rightarrow p_i} := {\psi'_i}^{\downarrow s_i}$
	\State $\psi'_{p_i} :=  \psi'_{p_i} \times \mu_{i \rightarrow p_i}$
\EndFor
\State \Return $\Psi, \Psi',\boldsymbol{\mu}$ where $\Psi=\{\psi_i|i\in V\}, \Psi'=\{\psi'_i|i\in V\}, \boldsymbol{\mu}=\{\mu_{i\rightarrow p_i}|i\in V-\{r\}\}$
\end{algorithmic}}{\footnotesize \par}

\protect\caption{\label{alg:Collect}\noun{Collect} algorithm}
\end{algorithm}

\begin{thm}
\label{thm:Collect}After running Algorithm \ref{alg:Collect} \noun{(Collect)}
over the nodes of a rooted covering join tree for $\phi=\prod_{k}\phi_{k}$,
we have that $\psi'_{i}=(\psi_{i}\times\prod_{j\in de(i)}\psi_{j})^{\downarrow\lambda(i)}$.
In particular, if $r$ is the root $\psi'_{r}=\phi^{\downarrow\lambda(r)}.$
\end{thm}
The theorem is an adaptation of Theorem 3.6. in \cite{Pouly2011a}
where the proof can be found. As a consequence of this theorem, we
can use the \noun{Collect} algorithm to solve the projection problem
provided that we are given a rooted covering join tree for the factorization
we would like to project and that the set of variables $X$ which
we want to project to is a subset of $\lambda(r)$.

\subsection{Assessing several projections of a factorized valuation}

Many times we are required to assess the projections of a single factorized
valuation to different subsets of variables. The corresponding problem
can be defined as follows 
\begin{problem}
\label{pbm:Marginalization-1}Let $(\Phi,U)$ be a valuation algebra,
$\phi_{1},\ldots,\phi_{n}$ be valuations in $\Phi$, and $\mathcal{T}=(V,E,\lambda,U)$
a rooted covering join tree for $\phi=\prod_{k}\phi_{k}$ . For all
$i\in V$, assess $\phi^{\downarrow\lambda(i)}.$

The \noun{Collect+Distribute} algorithm (Algorithm \ref{alg:Collect+Distribute})
shows how the result of the \noun{Collect} algorithm can be used to
assess the remaining projections by communicating messages down the
tree. The next result shows that the \noun{Collect+Distribute} algorithm
can be used to solve problem \ref{pbm:Marginalization-1}.\end{problem}
\begin{thm}
After running the \noun{Collect+Distribute} algorithm over the nodes
of a rooted covering join tree for $\phi=\prod_{k}\phi_{k}$, we have
that $\psi'_{i}=\phi^{\downarrow\lambda(i)}.$ 
\end{thm}
The theorem is a rewriting of Theorem 4.1 in \cite{Pouly2011a} where
the proof can be found. 

\begin{algorithm}
{\footnotesize{}\begin{algorithmic}[1]
\State $\Psi, \Psi', \boldsymbol{\mu} \leftarrow $ \Call{Collect}{$\Phi, \mathcal T$}
\ForAll{nodes $i$ of $\mathcal T$ except the root in a downward order }
	\State $\mu_{p_i \rightarrow i} := \big({\psi_{p_i}}\times\prod_{j\in ne(p_i)-\{i\}}\mu_{j \rightarrow p_i}\big)^{\downarrow s_i}$
	\State $\psi'_i := \psi'_i \times \mu_{p_i \rightarrow i}$
\EndFor
\State \Return $\Psi'$;
\end{algorithmic}}{\footnotesize \par}

\protect\caption{\label{alg:Collect+Distribute}\noun{Collect+Distribute} algorithm}
\end{algorithm}

\section{Finding solutions in valuation algebras: definitions and counterexamples\label{sec:GenericSolutions}}

In the previous section we have shown that the \noun{Collect} algorithm
can be used to assess one projection and that the \noun{Collect+Distribute}
algorithm can be used when many projections are needed. In this section
we focus on the \emph{solution finding problem} (SFP). 

The problem is of foremost importance, since it lies at the foundation
of dynamic programming \cite{Shenoy1996,Bertele1972}. Furthermore,
problems such as satisfiability, solving Maximum a Posteriori queries
in a probabilistic graphical models, or maximum likelihood decoding
are particular instances of the SFP. 

We start by formally defining the problem. Then we review the concept
of family of configuration extension sets which lies the foundation
of the theory of generic solutions described in \cite{Pouly2011c,Pouly2011a}.
Unfortunately, although the inspirational ideas and algorithms underlying
Pouly and Kohlas' work are correct, their formal development is not.
Thus, we end up the section providing two counter examples to one
of their fundamental theorems.

\subsection{The solution finding problem}

Up to now, the most general formalization of the SFP is the one provided
by \cite{Shenoy1996} and adapted by Pouly and Kohlas to the formal
framework of valuation algebras in Chapter 8 of \cite{Pouly2011a}.
 As in the projection assessment problem, in the SFP we are given
a set of valuations $\phi_{1},\ldots,\phi_{n}\in\Phi$ as input. However,
instead of a projection of its combination $\phi=\phi_{1}\times\ldots\times\phi_{n}$,
we are required to provide a tuple $\mathbf{x}$ with domain $d(\phi)$,
such that $\mathbf{x}$ is a \textbf{solution} for $\phi$. To give
a proper sense to the previous sentence we need to define the meaning
of ``being a solution''.\emph{ }The most general way in which we
can do this is by defining a family $c=\{c_{\phi}|\phi\in\Phi\}$
of solution sets. For each valuation $\phi\in\Phi$, the solution
set $c_{\phi}\subseteq\Omega_{d(\phi)}$. Now, $\mathbf{x}$ is considered
a solution for $\phi$ if and only if $\mathbf{x}\in c_{\phi}.$ We
say that the family of sets $c$ is a \emph{solution concept}. Now
we can formally define the SFP as follows
\begin{problem}[Solution Finding Problem (SFP)]
\noindent  Given a valuation algebra $(\Phi,U)$, a variable system
$\langle U,\Omega\rangle,$ a solution concept $c$, and a set of
valuations $\phi_{1},\ldots,\phi_{n}\in\Phi,$ the \emph{single SFP}
requests to find any $\mathbf{x}\in\Omega_{d(\phi)}$ such that $\mathbf{x}$
is a solution for $\phi=\phi_{1}\times\ldots\times\phi_{n}.$ The
\emph{partial SFP} receives the same input and requests to assess
a subset of the set of solutions $c_{\phi}$. The \emph{complete SFP
}receives the same input and requests to assess the full set of solutions
$c_{\phi}.$ 
\end{problem}

\subsection{Solving the solution finding problem by completing partial solutions\label{sec:Solving-SFP}}

Finding a solution to a big problem using dynamic programming amounts
to (1) breaking it into smaller problems, (2) start from an empty
solution, and (3) progressively complete this partial solution so
that it solves each of the smaller problems. Since we assume the existence
of a variable system, the empty solution will have no value assigned
to any variable. Then, each subproblem solved will complete the partial
solution by assigning values to some of the unassigned variables.
After the process is finished, all variables have a value assigned
and this complete assignment is a solution. 

In their works in 2011, Pouly and Kohlas \cite{Pouly2011a,Pouly2011c}
provide a formal foundation to dynamic programming. They present several
algorithmic schemas, and characterize the sufficient conditions for
their correctness. Their algorithms can be applied to previously uncovered
dynamic programming applications, such as solving systems of linear
equations. The most exhaustive presentation of Pouly and Kohlas' theory
is done in \cite{Pouly2011a}. We refer to that text as PK. For example
we use ``Lemma PK8.1'' to refer to Lemma 8.1 in \cite{Pouly2011a}. 

To formalize the process of completing a partial solution, they introduce
sets of extensions. Intuitively, given a tuple $\mathbf{x}$ with
domain $X$ and a valuation $\phi,$ the set of extensions of $\mathbf{x}$
to $\phi$, $W_{\phi}^{X}(\mathbf{x})$ contains those tuples that
we can concatenate to $\mathbf{x}$ to obtain a solution of $\phi.$
We say that $\mathbf{y}$ is an extension of $\mathbf{x}$ to $\phi$
whenever $\mathbf{y\in}W_{\phi}^{X}(\mathbf{x}).$ Following that,
the set of extension $W_{\phi}^{\emptyset}(\diamond)$ contains tuples
which are solutions of $\phi,$ that is $W_{\phi}^{\emptyset}(\diamond)\subseteq c_{\phi}.$
Although for solving the single and partial SFP it could be useful
that $W_{\phi}^{\emptyset}(\diamond)\subsetneq c_{\phi},$ in this
work we assume (with no impact on the results presented) that $W_{\phi}^{\emptyset}(\diamond)=c_{\phi}.$
If we define $c_{\phi}^{\downarrow X}=\{\conf{y}^{\downarrow X}\mid\conf{y}\in c_{\phi}\}$,
Lemma PK8.1 proves that $c_{\phi^{\downarrow X}}=c_{\phi}^{\downarrow X}.$
To simplify notation, we will always use $c_{\phi^{\downarrow X}}.$ 

We can constitute a family $\mbox{\ensuremath{\mathcal{W}}}$ containing
a set of extensions $W_{\phi}^{X}(\mathbf{x}),$ for each $\phi\in\Phi$,
each $X\subseteq d(\phi)$ and each $\mathbf{x}\in\Omega_{X}$. In
order for their algorithms to work Pouly and Kohlas' impose a condition
on this family, that basically states that every extension can be
calculated in two steps. Namely that for each $\phi\in\Phi$, for
each $X\subseteq Y\subseteq d(\phi)$ and for each $\conf{x}\in c_{\phi^{\downarrow X}}$,
we have that 

\begin{center}
$\conf{z}$ is an extension of $\conf{x}$ to $\phi$ iff $\begin{array}{l}
\mbox{\ensuremath{\mathbf{z}^{\downarrow Y-X}} is an extension of \textbf{\ensuremath{\mathbf{x}}} to \ensuremath{\phi^{\downarrow Y}}, }\mbox{and}\\
\mbox{\ensuremath{\mathbf{z}^{\downarrow d(\phi)-Y}} is an extension of \ensuremath{\langle\mathbf{x},\mathbf{z}^{\downarrow Y-X}\rangle} to \ensuremath{\phi.} }
\end{array}$
\par\end{center}

More formally, 
\begin{defn}[Extension system]
\label{def:FCES}A family of extension sets $\mbox{\ensuremath{\mathcal{W}}}=\{W_{\phi}^{X}(\mathbf{x})\subseteq\Omega_{d(\phi)-X}|\phi\in\Phi,X\subseteq d(\phi),\mathbf{x}\in\Omega_{X}\}$
constitutes an \emph{extension system}\footnote{Note that Pouly and Kohlas' do never formally introduce extension
systems. Our definition here is slightly less constraining than their
informal definition. All of the counterexamples defined later do also
fulfill their informal definition. }\emph{ }if and only if 

\begin{equation}
W_{\phi}^{X}(\conf{x})=\{\langle\conf{y},\conf{z}\rangle|\conf{y}\in W_{\phi^{\downarrow Y}}^{X}(\conf{x})\mbox{ and }\conf{z}\in W_{\phi}^{Y}(\langle\conf{x},\conf{y}\rangle)\}\ \ \ \ \ \ \ \ \ \ \ \ \ \ \forall\mathbf{x\in}c_{\phi^{\downarrow X}}.\label{eq:FCES-condition}
\end{equation}
\end{defn}
\begin{example}
\label{ex:FCESValAlgBooleanFunctions}For the valuation algebra of
indicator functions introduced in example \ref{ex:ValAlgBooleanFunctions},
we can define a family of sets $\mathcal{W}=\{W_{\phi}^{X}(\mathbf{x})|\phi\in\Phi,X\subseteq d(\phi),\mathbf{x}\in\Omega_{X}\}$,
with each 
\begin{eqnarray}
W_{\phi}^{X}(\mathbf{x}) & = & \{\mathbf{y}\in\Omega_{d(\phi)-X}\mid\phi(\langle\mathbf{x},\mathbf{y}\rangle)=\phi^{\downarrow X}(\mathbf{x})\}\label{eq:FCESValAlgBooleanFunctions}\\
 & = & \{\mathbf{y}\in\Omega_{d(\phi)-X}\mid\phi(\langle\mathbf{x},\mathbf{y}\rangle)=\max_{\mathbf{z}\in\Omega_{d(\phi)-X}}\phi(\langle\mathbf{x,z}\rangle)\}\\
 & = & \argmax{}_{\mathbf{z}\in\Omega_{d(\phi)-X}}\phi(\langle\mathbf{x,z}\rangle).
\end{eqnarray}
As proven in \cite{Pouly2011a} $\mathcal{W}$ satisfies equation
\ref{eq:FCES-condition} and thus constitutes an extension system. 
\end{example}
Now that we have defined what it means to be an extension, we can
now formally define what we mean by a completion.

\begin{defn}[Completion]
 Given a valuation $\phi,$ a domain $X,$ and a configuration $\mathbf{x}\in\Omega_{X},$
we say that $\mathbf{y}$ is a completion of $\mathbf{x}$ to $\phi$
if, and only if, $d(\mathbf{y})=d(\phi),$ $\mathbf{y}^{\downarrow X}=\mathbf{x}$
and $\mathbf{y}^{\downarrow d(\phi)-X}\in W_{\phi}^{d(\phi)\cap X}(\mathbf{x}^{\downarrow d(\phi)\cap X}).$
We define the set of completions of $A\subseteq\Omega_{X}$ to $\phi$
as $CO(A,\phi)=\{\langle\conf{x},\mathbf{z}\rangle\mid\conf{x}\in A\mbox{ and }\conf{z}\in W_{\phi}^{d(\phi)\cap X}(\conf{x}^{\downarrow d(\phi)\cap X})\}\subseteq\Omega_{d(\phi)\cup X}.$ 
\end{defn}
Note that $c_{\phi}=CO(\{\diamond\},\phi).$

\subsection{A fundamental theorem and two counterexamples }

Based on the former definitions, Pouly and Kohlas state the following
theorem 
\begin{thm}[Theorem PK8.1]
\label{thm:8.1} For any valuation $\phi\in\Phi$ and any $X,Y\subseteq d(\phi)$,
we have
\begin{equation}
c_{\phi^{\downarrow X\cup Y}}=CO(c_{\phi^{\downarrow X}},\phi^{\downarrow Y}).\label{eq:thm:8.1}
\end{equation}

\end{thm}
Unfortunately, the theorem is not correct. To understand the theorem
and what goes wrong we can concentrate in the simpler particular case
in which $X\cup Y=d(\phi).$
\begin{thm}[Simplified version of Theorem PK8.1 in \cite{Pouly2011a}]
\label{thm:8.1-2} For any valuation $\phi\in\Phi$ and any $X,Y$
such that $X\cup Y=d(\phi)$, we have {\footnotesize{}
\begin{equation}
c_{\phi}=CO(c_{\phi^{\downarrow X}},\phi^{\downarrow Y}).\label{eq:thm:8.1-2}
\end{equation}
}{\footnotesize \par}
\end{thm}
In the theorem, $X$ and $Y$ represent a possible way of breaking
the problem in two pieces, namely $\phi^{\downarrow Y}$ and $\phi^{\downarrow X}$.
Basically the theorem states that any solution of $\phi$ can be assessed
by taking a solution to the smaller problem $\phi^{\downarrow X},$
and then completing it to the other smaller problem $\phi^{\downarrow Y}.$
Furthermore, it states that each of the configurations built following
that procedure is in fact a  solution of $\phi$.

Next, we will provide a counterexample that disproves the theorem.

\subsubsection{First counterexample}

The counterexample is based in the valuation algebra of indicator
functions introduced in example \ref{ex:ValAlgBooleanFunctions} with
the extension system introduced in example \ref{ex:FCESValAlgBooleanFunctions}. 

\begin{counter}\label{cnt:Boolean}

Theorem \ref{thm:8.1-2} does not hold. 

\end{counter}
\begin{proof}
Let $x,y$ be two Boolean variables and $\phi$ the indicator function
with $d(\phi)=\{x,y\}$ and 
\[
\phi(\mathbf{z})=\mathbb{1}_{\mathbf{z}(x)=\mathbf{z}(y)}=\begin{cases}
1 & \mbox{if \ensuremath{\mathbf{z}(x)=\mathbf{z}(y)}},\\
0 & \mbox{otherwise.}
\end{cases}
\]

Taking $X=\{x\}$, and $Y=\{y\}$ we will see that Theorem \ref{thm:8.1-2}
does not hold. To see why, we will first assess the set of solutions
for our valuation, namely $c_{\phi}.$ Then, we will assess the set
of solutions that can be found by completing a partial solution to
$\phi^{\downarrow X},$ as suggested in the right hand side of Equation
\ref{eq:thm:8.1-2}. We will see that those two sets are different,
contradicting Theorem \ref{thm:8.1-2}. 

By definition, the set of  solutions for our valuation, $c_{\phi}=W_{\phi}^{\emptyset}(\diamond).$
Applying equation \ref{eq:FCESValAlgBooleanFunctions} we have that
\begin{equation}
c_{\phi}=\{\langle\mathbf{x},\mathbf{y}\rangle\in\Omega_{\{x,y\}}\mid\phi(\langle\mathbf{x},\mathbf{y}\rangle)=\phi^{\downarrow\emptyset}(\diamond)\}.\label{eq:counterExampleIntermediate}
\end{equation}
Now, we can assess  $\phi^{\downarrow\emptyset}(\diamond)=\max_{\mathbf{x},\mathbf{y}}\phi(\langle\diamond,\langle\mathbf{x},\mathbf{y}\rangle\rangle)=\max_{\mathbf{x},\mathbf{y}}\phi(\langle\mathbf{x},\mathbf{y}\rangle)=1$,
and from the definition of $\phi$ and equation \ref{eq:counterExampleIntermediate}
we have that $c_{\phi}=\{\{x\mapsto0,y\mapsto0\},\{x\mapsto1,y\mapsto1\}\}$,
where $\{x\mapsto0,y\mapsto0\}$ is the tuple assigning value $0$
to variables $x$ and $y.$ 

Now we will assess $CO(c_{\phi^{\downarrow X}},\phi^{\downarrow Y}),$
to see that they do not coincide. Since $d(\phi^{\downarrow Y})=Y$,
we have that $d(\phi^{\downarrow Y})\cap X=\emptyset$, thus we can
use the definition of set of completions to get
\begin{eqnarray*}
CO(c_{\phi^{\downarrow X}},\phi^{\downarrow Y}) & = & \{\langle\mathbf{x,y}\rangle\in\Omega_{X\cup Y}\mid\mathbf{x}\in c_{\phi^{\downarrow X}}\mbox{ and }\mathbf{y}\in W_{\phi^{\downarrow Y}}^{d(\phi^{\downarrow Y})\cap X}(x^{d(\phi^{\downarrow Y})\cap X})\}\\
 & = & \{\langle\mathbf{x,y}\rangle\in\Omega_{X\cup Y}\mid\mathbf{x}\in c_{\phi^{\downarrow X}}\mbox{ and }\mathbf{y}\in W_{\phi^{\downarrow Y}}^{\emptyset}(\diamond)\}\\
 & = & \{\langle\mathbf{x,y\rangle}\in\Omega_{X\cup Y}\mid\mathbf{x}\in c_{\phi^{\downarrow X}}\mbox{ and }\mathbf{y}\in c_{\phi^{\downarrow Y}}\}.
\end{eqnarray*}
We can now assess $c_{\phi^{\downarrow X}}$ as $c_{\phi^{\downarrow X}}=c_{\phi}^{\downarrow X}=\{\mathbf{z}^{\downarrow X}\mid\mathbf{z}\in c_{\phi}\}=\{\{x\mapsto0\},\{x\mapsto1\}\}=\Omega_{X}$
and $c_{\phi^{\downarrow X}}$ as $c_{\phi^{\downarrow Y}}=c_{\phi}^{\downarrow Y}=\{\mathbf{z}^{\downarrow Y}\mid\mathbf{z}\in c_{\phi}\}=\{\{y\mapsto0\},\{y\mapsto1\}\}=\Omega_{Y}$.
Hence, $CO(c_{\phi^{\downarrow X}},\phi^{\downarrow Y})=\{\langle\mathbf{x,y}\rangle\in\Omega_{X\cup Y}\mid\mathbf{x}\in\Omega_{X}\mbox{ and }\mathbf{y}\in\Omega_{Y}\}=\Omega_{X\cup Y}$,
and from here we have that $CO(c_{\phi^{\downarrow X}},\phi^{\downarrow Y})\neq c_{\phi}$
contradicting equation \ref{eq:thm:8.1}.
\end{proof}

\subsubsection{Second counterexample}

One may think that theorem \ref{thm:8.1-2} would become true by requiring
that $\phi=\phi_{X}\times\phi_{Y}$ for some $\phi_{X},\phi_{Y}\in\Phi$
such that $d(\phi_{X})=X$ and $d(\phi_{Y})=Y$.

Nonetheless, the following counterexample shows that as long as the
extension system is not related to operations in the valuation algebra
we can create a counterexample that fulfils the above requirement.

\begin{counter}\label{cnt:Boolean2}

Theorem \ref{thm:8.1-2} with the additional hypothesis that $\phi=\phi_{X}\times\phi_{Y}$
for some $\phi_{X},\,\phi_{Y}\in\Phi$ such that $X=d(\phi_{X})$
and $Y=d(\phi_{Y})$ still does not hold.

\end{counter}
\begin{proof}
Take any $\phi_{X},\phi_{Y}\in\Phi$ such that $d(\phi_{X})=X$ and
$d(\phi_{Y})=Y$ and $\phi=\phi_{X}\times\phi_{Y}$. As we did in
the first counterexample take $X=\{x\}$, and $Y=\{y\}.$ Now, instead
of using the extension system introduced in example \ref{ex:FCESValAlgBooleanFunctions},
we define $\mathcal{\overline{W}}$ as follows: $\overline{W}_{\xi}^{Z}(\alpha)=W_{\eta^{\downarrow d(\phi)}}^{Z}(\alpha)$
where $\eta\in\Phi$ is the indicator function $\eta(\mathbf{z})=\mathbb{1}_{\mathbf{z}(x)=\mathbf{z}(y)}$
which we used for our former counter example and $W$ is the extension
system in example \ref{ex:FCESValAlgBooleanFunctions}. It is important
to remark that we are defining the sets of extensions in terms of
$\eta.$ Thus, for any $\xi,$ the set of extensions $\overline{W}_{\xi}^{Z}(\alpha)$
depends on $\alpha$ and the domain of $\xi,$ but it is the same
for any two valuations $\xi$ and $\xi'$ with the same domain. 

Notice that $\mathcal{\overline{W}}$ is well defined and does satisfy
equation \ref{eq:FCES-condition}, thus $\overline{\mathcal{W}}$
is an extension system. We refer to the solutions of this new extension
system as $\overline{c}$ and to the completions as $\overline{CO},$
while we keep using $W,$ $c$ and $CO$ for the extension system
introduced in example \ref{ex:FCESValAlgBooleanFunctions}.

Now, following exactly the same reasoning as in the previous counterexample,
we get $\overline{c}_{\phi}=c_{\eta}=\{\{x\mapsto0,y\mapsto0\},\{x\mapsto1,y\mapsto1\}\},$
whilst 
\begin{eqnarray*}
\overline{CO}(\overline{c}_{\phi^{\downarrow X}},\phi^{\downarrow Y}) & = & \{\langle\mathbf{x,y}\rangle\in\Omega_{X\cup Y}\mid\mathbf{x}\in\overline{c}_{\phi^{\downarrow X}}\mbox{ and }\mathbf{y}\in\overline{W}_{\phi^{\downarrow Y}}^{d(\phi^{\downarrow Y})\cap X}(\mathbf{x}^{d(\phi^{\downarrow Y})\cap X})\}\\
 & = & \{\langle\mathbf{x,y}\rangle\in\Omega_{X\cup Y}\mid\mathbf{x}\in c_{\eta^{\downarrow X}}\mbox{ and }\mathbf{y}\in W_{\eta^{\downarrow Y}}^{d(\eta^{\downarrow Y})\cap X}(\mathbf{x}^{d(\eta^{\downarrow Y})\cap X})\}\\
 & = & CO(c_{\eta^{\downarrow X}},\eta^{\downarrow Y})=\Omega_{X\times Y}.
\end{eqnarray*}

Therefore we get $\overline{c}_{\phi}\neq\overline{CO}(\overline{c}_{\phi^{\downarrow X}},\phi^{\downarrow Y})$,
which contradicts theorem \ref{thm:8.1-2} again.
\end{proof}

\section{Impact of the counterexamples\label{sec:Impact}}

In this section we consider the overall impact of the disproved theorem
on Pouly and Kohlas' theory. The theory in chapter 8 of \cite{Pouly2011a}
has two main parts. In the first one (section 8.2), they propose and
give sufficient conditions to some algorithms for computing solutions.
In the second one (section 8.4) they analyze which algorithms can
be applied in the case of optimization problems (valuation algebras
induced by semirings with idempotent addition). In the following we
review the main results of each section and how the problem detected
with Theorem PK8.1 affects them.
\begin{algorithm}
{\footnotesize{}
\begin{algorithmic}[1]
\algnewcommand\algorithmicinput{\textbf{Input:}} 
\algnewcommand\Input{\item[\algorithmicinput]}
\Input The set of projections of $\phi$, $\{\phi^{\downarrow \lambda(i)}|i\in V\}$ (usually a result of \textsc{Collect} + \textsc{Distribute}).
\State $c \leftarrow \{\diamond\}$
\ForAll{nodes $i$ of $\mathcal T$ in a downward order}
	\State $c \leftarrow CO(c,\phi^{\downarrow \lambda(i)})$
\EndFor
\State \Return c;
\end{algorithmic}}{\footnotesize \par}

\protect\caption{\label{alg:ExtendAll}\noun{\EGPName\  }algorithm}
\end{algorithm}

\begin{algorithm}
{\footnotesize{}\begin{algorithmic}[1]
\algnewcommand\algorithmicinput{\textbf{Input:}} 
\algnewcommand\Input{\item[\algorithmicinput]}
\Input The set $\{\psi'_i|i\in V\}$ that results of \textsc{Collect}. 
\State $c \leftarrow \{\diamond\}$
\ForAll{nodes $i$ of $\mathcal T$ in a downward order}
	\State $c \leftarrow CO(c,\psi'_i)$
\EndFor
\State \Return c;
\end{algorithmic}}{\footnotesize \par}

\protect\caption{\label{alg:ExtendSome}\ESName\ algorithm}
\end{algorithm}

\begin{algorithm}[t]
{\footnotesize{}\begin{algorithmic}[1]
\algnewcommand\algorithmicinput{\textbf{Input:}} 
\algnewcommand\Input{\item[\algorithmicinput]}
\Input The set $\{\psi'_i|i\in V\}$ that results of \textsc{Collect}.
\State $\conf{x}  \leftarrow \diamond$
\ForAll{nodes $i$ of $\mathcal T$ in a downward order}
	\State $\mathbf{x} \leftarrow $  A completion of $\mathbf{x}$ to $\psi'_i$
\EndFor
\State \Return $\conf{x}$;
\end{algorithmic}}{\footnotesize \par}

\protect\caption{\label{alg:ExtendOne}\SESName\  algorithm}
\end{algorithm}

\subsection{\label{sec:PK-Sufficient-Conditions}Generic algorithms to compute
solutions and their sufficient conditions}

In the first part of the theory, three different algorithms are presented.
The first algorithm computes a set of solutions by (i) assessing the
projections using the \noun{Collect+Distribute} algorithm, and then
(ii) using those projections to assess a set of solutions. Algorithm
\ref{alg:ExtendAll}, called \EGPName, shows the procedure and is
equivalent to algorithm PK8.1. The algorithm is proven to solve the
complete SFP for any extension system as a byproduct of Lemma PK8.2. 

The second algorithm computes some solutions by (i) running \noun{Collect
}to assess the subtree projections and then (ii) using the subtree
projections to assess a set of solutions. Algorithm \ref{alg:ExtendSome},
named\noun{ }\ESName\noun{,} illustrates how the subtree projections
are combined to assess a set of solutions and is equivalent to algorithm
PK8.2. The sufficient conditions for this algorithm to solve the partial
SFP are provided by Theorem PK8.2. They are
\begin{itemize}
\item {[}CPK1{]} Configuration extension sets need to be always non-empty
and 
\item {[}CPK2{]} For each $\xi_{1},\xi_{2}\in\Phi$, with domains $X$ and
$Y$ respectively, each $X\subseteq Z\subseteq X\cup Y$ and each
$\mathbf{x}\in\Omega_{Z},$ we have 
\[
W_{\xi_{2}}^{Z\cap Y}(\mathbf{x}^{\downarrow Z\cap Y})\subseteq W_{\xi_{1}\times\xi_{2}}^{Z}(\mathbf{x}).
\]

\end{itemize}
The conditions for this algorithm to solve the complete SFP is given
by Theorem PK8.3 and is 
\begin{itemize}
\item {[}CPK3{]} For each $\xi_{1},\xi_{2}\in\Phi$, with domains $X$ and
$Y$ respectively, each $X\subseteq Z\subseteq X\cup Y$ and each
$\mathbf{x}\in\Omega_{Z},$ we have 
\[
W_{\xi_{2}}^{Z\cap Y}(\mathbf{x}^{\downarrow Z\cap Y})=W_{\xi_{1}\times\xi_{2}}^{Z}(\mathbf{x}).
\]

\end{itemize}
The third algorithm finds one solution by (i) running \noun{Collect
}and then (ii) using the subtree projections to assess a single solution.
Algorithm \ref{alg:ExtendOne}, called \SESName\  shows how a single
solution is assessed and is equivalent to Algorithm PK8.3. The sufficient
condition for this algorithm to solve the single SFP are again CPK1
and CPK2 provided by Theorem PK8.2. 

The proofs of Lemma PK8.2, Theorem PK8.2 and Theorem PK8.3 relied,
either in a direct or indirect way, on Theorem PK8.1. Thus, for each
of these results we need to determine whether they still hold (and
only a new proof needs to be found) or whether they no longer hold.
Later we will show that whilst Theorem PK8.2 and PK8.3 are correct
(we will provide an alternative proof), Lemma PK8.2 requires an additional
condition. The impact of the counterexamples on the theory is summarized
in Table \ref{tab:Impact}. 

In this paper we repair the theory by (i) providing corrected proofs
for those results that are true but incorrectly proven and (ii) identifying
the sufficient condition required for \EGPName\  to work. Furthermore,
we show that the sufficient conditions identified for the algorithms
are not only sufficient but also necessary. 

\begin{table}
\begin{centering}
{\small{}}%
\begin{tabular}{|c|c|>{\centering}m{2cm}|c|>{\centering}m{3cm}|}
\hline 
\textbf{\scriptsize{}PK Result} & \textbf{\scriptsize{}Algorithm } & \textbf{\scriptsize{}Suff. cond. } & \textbf{\scriptsize{}Solutions} & \textbf{\scriptsize{}Impact}\tabularnewline
\hline 
\hline 
{\scriptsize{}Lemma PK8.2} & \noun{\scriptsize{}\ref{alg:ExtendAll}} & {\scriptsize{}None} & {\scriptsize{}All} & {\scriptsize{}False. Necessary condition required. }\tabularnewline
\hline 
{\scriptsize{}Theorem PK8.2} & \noun{\scriptsize{}\ref{alg:ExtendOne}} & {\scriptsize{}CPK1, CPK2} & {\scriptsize{}One} & {\scriptsize{}True, but a correct proof is required. }\tabularnewline
\hline 
{\scriptsize{}Theorem PK8.2} & \noun{\scriptsize{}\ref{alg:ExtendSome}} & {\scriptsize{}CPK1, CPK2} & {\scriptsize{}Some} & {\scriptsize{}True, but a correct proof is required.}\tabularnewline
\hline 
{\scriptsize{}Theorem PK8.3} & \noun{\scriptsize{}\ref{alg:ExtendSome}} & {\scriptsize{}CPK1, CPK3} & {\scriptsize{}All} & {\scriptsize{}True, but a correct proof is required.}\tabularnewline
\hline 
\end{tabular}
\par\end{centering}{\small \par}

\protect\caption{\label{tab:Impact}Impact of the counterexample on Pouly and Kohlas
results about the sufficient conditions of the generic algorithms}
\end{table}

\subsection{Impact on sufficient conditions on optimization problems}

After discussing generic algorithms, Pouly and Kohlas particularize
their results to optimization problems in section PK8.4. There it
is shown that for any valuation algebra induced by a selective\footnote{Although Pouly and Kohlas use the term totally ordered idempotent
semiring, in this work we follow the notation in \cite{Gondran2008}
and use selective semiring for the very same concept. See corollary
\ref{cor:SelectiveSemiring} in appendix.} semiring it is possible to define an extension system. They rely
on Lemma PK8.2 to prove that no additional condition is needed to
guarantee the correctness of \EGPName. Since we have seen that Lemma
PK8.2 is flawn, we need to revise that conclusion. 

Furthermore, they show that the extension system fulfills the sufficient
condition in Theorem PK8.2, thus enabling the usage of \SESName\ 
to solve the single SFP and of \ESName\  to solve the partial SFP.
Furthermore, if the semiring is also strict monotonic then the extension
system satisfies the sufficient conditions of Theorem PK8.3, enabling
the usage of \ESName\  to solve the complete SFP. Since Theorem PK8.2
and Theorem PK8.3 are correct, only the conclusions arising from Lemma
PK8.2 should be revised. 

\begin{table}
\begin{centering}
{\small{}}%
\begin{tabular}{|c|>{\centering}m{2.3cm}|c|>{\centering}m{3cm}|}
\hline 
\textbf{\scriptsize{}Algorithm } & \textbf{\scriptsize{}Semiring } & \textbf{\scriptsize{}Solutions} & \textbf{\scriptsize{}Impact}\tabularnewline
\hline 
\hline 
\noun{\scriptsize{}\ref{alg:ExtendAll}} & {\scriptsize{}None} & {\scriptsize{}All} & {\scriptsize{}Incorrect.}\tabularnewline
\hline 
\noun{\scriptsize{}\ref{alg:ExtendOne}} & {\scriptsize{}None} & {\scriptsize{}One} & {\scriptsize{}Correct.}\tabularnewline
\hline 
\noun{\scriptsize{}\ref{alg:ExtendSome}} & {\scriptsize{}None} & {\scriptsize{}Some} & {\scriptsize{}Correct.}\tabularnewline
\hline 
\noun{\scriptsize{}\ref{alg:ExtendSome}} & {\scriptsize{}Strict monotonic} & {\scriptsize{}All} & {\scriptsize{}Correct but can be weakened}\tabularnewline
\hline 
\end{tabular}
\par\end{centering}{\small \par}

\protect\caption{\label{tab:ImpactOpt}Impact of the counterexample on Pouly and Kohlas
results about the necessary conditions of the algorithms for optimization
problems. On those problems, semirings are commutative and selective.}
\end{table}

In this paper we improve the characterization of the algorithms for
optimization problems given by Pouly and Kohlas by (i) providing a
necessary condition and a sufficient condition on the semiring which
guarantees the correctness of algorithm \EGPName\  and (ii) weakening
the sufficient condition under which \ESName\  is guaranteed to solve
the complete SFP and showing that the condition is also necessary.

\section{Correcting the theory of generic solutions in valuation algebras\label{sec:NewConditions}}

In this section we concentrate on providing sufficient conditions
for the three generic algorithms presented above. Furthermore, we
also show that for some of the algorithms, these conditions are necessary.
We start by proving a lemma that lies at the foundation of the proofs
of the results to come. Then, we introduce two different conditions,
namely projective completability and piecewise completability, which
can be imposed to an extension system and we study the relationship
between them. Then, we prove that projective completability is a sufficient
and necessary condition for algorithm \EGPName\  to find all  solutions.
After that we study how piecewise completability determines the correctness
of the \ESName\ \noun{ }and \SESName\  algorithms. We close the
section by explaining how those result in \cite{Pouly2011a} which
were correct can be proven from the results presented here.

\subsection{The covering join tree decomposition lemma}

Our first objective is to characterize subsets of valuations which
are well behaved with respect of the operations of the valuation algebra. 
\begin{defn}
A subset of valuations $\Xi\subseteq\Phi$ is \emph{projection-closed}
if for each $\phi\in\Xi,$ and each $X\subseteq d(\phi)$, $\phi^{\downarrow X}\in\Xi$.
A subset of valuations $\Xi\subseteq\Phi$ is \emph{combination-breakable}
if for each $\phi\in\Xi$, such that $\phi=\xi_{1}\times\xi_{2},$
we have that both $\xi_{1},\xi_{2}\in\Xi.$ 
\end{defn}
If a subset of valuations is projection-closed we can safely project
a valuation in the subset and we know we will get another valuation
in the subset. A subset of valuations is combination-breakable if
whenever we can factorize a valuation in the subset as a combination
of two other valuations we know that each of the components is guaranteed
to be in the subset. Note that this does not imply that if we take
two valuations from the subset its product will be in the subset.
A trivial example of projection-closed and combination-breakable set
of valuations is the set of all valuations $\Phi.$ 

Next, we introduce the main result of this section proving that for
any node $i\in V$ in a join tree $(V,E,\lambda,U)$ under reasonable
conditions on $X$, we can express the projection $\left(\phi_{1}\times\cdots\times\phi_{n}\right)^{\downarrow(X\cup\lambda(i))}$
as a product of two valuations, one of them with scope $X$ and the
other one with scope $\lambda(i).$ 
\begin{figure}
\begin{centering}
\usetikzlibrary{arrows,shapes}
\tikzstyle{vertex}=[solid, ellipse, black, draw, minimum width=1cm, minimum height=1cm]
\tikzstyle{edge} = [solid,black,draw,line width=1pt,-]
\input{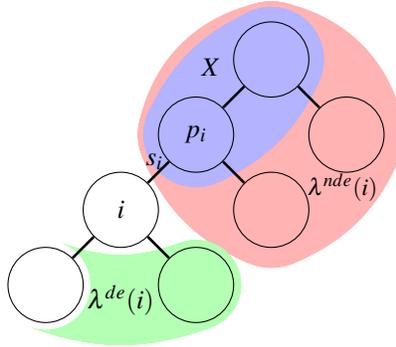}
\par\end{centering}

\protect\caption{Visualizing the decomposition lemma}
\label{fig:Visualizing-the-decomposition}
\end{figure}

The conditions on $X$ are that it should cover the separator $s_{i}$
and that all its variables should appear in the non-descendants of
$i.$ In order to formalize the condition for each node $i\in V,$
we define $\lambda^{de}(i)$ as the set of variables that appear in
the scope of the descendants of $i$, namely $\lambda^{de}(i)=\bigcup_{j\in de(i)}\lambda(j)$.
Furthermore we define $\lambda^{nde}(i)$ as the set of variables
that appear in the scope of the non-descendants of $i,$ namely $\lambda^{nde}(i)=\bigcup_{j\in nde(i)}\lambda(j)$.
Figure \ref{fig:Visualizing-the-decomposition} shows $X$ in blue,
$\lambda^{de}(i)$ in green and $\lambda^{nde}(i)$ in red in a simple
example to help understanding the notation and the conditions on the
lemma. 
\begin{lem}
\label{lem:Decomposition}Let $(\Phi,U)$ be a valuation algebra.
Let $\Xi\subseteq\Phi$ be a subset of valuations projection-closed
and combination-breakable. Let $\phi\in\Xi,$ $\phi=\phi_{1}\times\cdots\times\phi_{n}.$
For any node $i$ of $\mathcal{T},$ and any domain $X\subseteq\lambda^{nde}(i),$
such that $s_{i}\subseteq X,$ we have that $\phi^{\downarrow(X\cup\lambda(i))}$
factorizes as \textup{
\[
\phi^{\downarrow(X\cup\lambda(i))}=\alpha\times\beta,
\]
}\textup{\emph{with }}\textup{$\alpha\in\Xi,$ $d(\alpha)=X$ }\textup{\emph{and}}\textup{
$\beta\in\Xi,$ $d(\beta)=\lambda(i).$ }\\
\textup{\emph{Concretely }}$\alpha=\left(\prod_{j\in nde(i)}\psi_{j}\right)^{\downarrow X}$
and $\beta=\left(\psi_{i}\times\prod_{j\in de(i)}\psi_{j}\right)^{\downarrow\lambda(i)}.$\end{lem}
\begin{proof}
From definition \ref{def:CoveringJoinTree}, we have that $\phi=\prod_{i\in V}\psi_{i}.$
We can factorize $\phi$ as $\phi=\eta_{1}\times\psi_{i}\times\eta_{2},$
where $\eta_{1}=\prod_{j\in nde(i)}\psi_{j}$ with $d(\eta_{1})=\bigcup_{j\in nde(i)}d(\psi_{j})$
and $\eta_{2}=\prod_{j\in de(i)}\psi_{j}.$ By equation \ref{eq:lambdaDIsProduct}
from the appendix, we have that $d(\eta_{2})=\lambda^{de}(i).$ 

Applying the factorization we have that 
\begin{eqnarray*}
\phi^{\downarrow(X\cup\lambda(i))} & = & \left(\eta_{1}\times\psi_{i}\times\eta_{2}\right)^{\downarrow X\cup\lambda(i)}.
\end{eqnarray*}
Since $X\subseteq\lambda^{nde}(i)$, we have that $X\cup\lambda(i)\subseteq\lambda^{nde}(i)\cup\lambda(i)$,
and by axiom A4
\[
\phi^{\downarrow(X\cup\lambda(i))}=\left(\left(\eta_{1}\times\psi_{i}\times\eta_{2}\right)^{\downarrow\lambda^{nde}(i)\cup\lambda(i)}\right)^{\downarrow X\cup\lambda(i)}.
\]
Now $\lambda^{nde}(i)\cup\lambda(i)$ covers both $d(\eta_{1})$ and
$d(\psi_{i})$ by the covering property, so we can apply axiom A5
to get  
\[
\phi^{\downarrow(X\cup\lambda(i))}=\left((\eta_{1}\times\psi_{i})\times\eta_{2}^{\downarrow(\lambda^{nde}(i)\cup\lambda(i))\cap\lambda^{de}(i)}\right)^{\downarrow X\cup\lambda(i)}.
\]
Analyzing the domain where $\eta_{2}$ is projected to we find that
\begin{eqnarray*}
\left(\lambda^{nde}(i)\cup\lambda(i)\right)\cap\lambda^{de}(i) & = & \left(\lambda^{nde}(i)\cap\lambda^{de}(i)\right)\cup\left(\lambda(i)\cap\lambda^{de}(i)\right)\\
 & = & \left(\lambda(i)\cap\lambda^{de}(i)\right)=\bigcup_{j\in ch(i)}s_{j}
\end{eqnarray*}
 where the first equality distributes the intersection, the second
one applies that by equation \ref{eq:lambdaLargerThanIntersection}
from the appendix we know that $\lambda^{nde}(i)\cap\lambda^{de}(i)\subseteq\lambda(i)\cap\lambda^{de}(i)$
and the third one uses equation \ref{eq:intersectionWithSubtree}
also found at the appendix. Replacing in the expression above we get 

\[
\phi^{\downarrow(X\cup\lambda(i))}=\left(\eta_{1}\times\psi_{i}\times\eta_{2}^{\downarrow\bigcup_{j\in ch(i)}s_{j}}\right)^{\downarrow X\cup\lambda(i)}.
\]
Since $\bigcup_{j\in ch(i)}s_{j}\subseteq\lambda(i)$ and $d(\psi_{i})\subseteq\lambda(i)$
we can apply again axiom A5, 
\[
\phi^{\downarrow(X\cup\lambda(i))}=\eta_{1}^{\downarrow(X\cup\lambda(i))\cap d(\eta_{1})}\times(\psi_{i}\times\eta_{2}^{\downarrow\bigcup_{j\in ch(i)}s_{j}}).
\]
From equation \ref{eq:lambdaNDIsProduct} from the appendix, we have
that $d(\eta_{1})=\lambda^{nde}(i)$ and then 
\[
\phi^{\downarrow(X\cup\lambda(i))}=\eta_{1}^{\downarrow(X\cup\lambda(i))\cap\lambda^{nde}(i)}\times\psi_{i}\times\eta_{2}^{\downarrow\bigcup_{j\in ch(i)}s_{j}}.
\]
Distributing the intersection, we have that $(X\cup\lambda(i))\cap\lambda^{nde}(i)=(X\cap\lambda^{nde}(i))\cup(\lambda(i)\cap\lambda^{nde}(i)).$
In the lemma we required that $X\subseteq\lambda^{nde}(i)$, and from
here $X\cap\lambda^{nde}(i)=X$. On the other hand by equation \ref{eq:separatorIsIntersectionWithND},
we have that $s_{i}=(\lambda(i)\cap\lambda^{nde}(i))$ and since $s_{i}\subseteq X,$
we get that
\[
\phi^{\downarrow(X\cup\lambda(i))}=\eta_{1}^{\downarrow X}\times\psi_{i}\times\eta_{2}^{\downarrow\bigcup_{j\in ch(i)}s_{j}}
\]
and applying axiom A5 one last time, this time to join instead of
to split, we get 
\[
\phi^{\downarrow(X\cup\lambda(i))}=\eta_{1}^{\downarrow X}\times\left(\psi_{i}\times\eta_{2}\right)^{\downarrow d(\psi_{i})\cup\bigcup_{j\in ch(i)}s_{j}}.
\]
Finally, by equation \ref{eq:MinimalLabeling} we have that $\lambda(i)=d(\psi_{i})\cup\bigcup_{j\in ch(i)}s_{j}.$
So, we directly identify that $\phi^{\downarrow X\cup\lambda(i)}$
factorizes as $\alpha\times\beta,$ where $\alpha=\eta_{1}^{\downarrow X}$
and $\beta=\left(\psi_{i}\times\eta_{2}\right)^{\downarrow\lambda(i)},$
with $d(\alpha)=X$ and $d(\beta)=\lambda(i).$ 

Note that since $\phi\in\Xi$, and $\Xi$ is projection-closed, $\phi^{\downarrow X\cup\lambda(i)}\in\Xi.$
Now, since $\phi^{\downarrow X\cup\lambda(i)}=\alpha\times\beta$,
and $\Xi$ is combination-breakable we have that $\alpha\in\Xi,$
and $\beta\in\Xi$.
\end{proof}

\subsection{Completability properties of extension systems.}

In this section we define some properties which will allow us to characterize
under which conditions the different algorithms work. Intuitively,
these properties impose conditions under which the solution to a ``simpler''
problem can be completed to obtain a solution to a ``more complex''
problem. 

In this section, let $(\Phi,U)$ be a valuation algebra, $\mathcal{W}$
a extension system and $\Xi\subseteq\Phi$ be a subset of valuations
projection-closed and combination-breakable.

\begin{defn}[Projective completability]
 We say that \emph{projective completability (on products) }holds
on $(\Phi,U),$ $\mathcal{W},$ and $\Xi$ if for each valuation $\phi\in\Xi$
such that $\phi=\xi_{1}\times\xi_{2}$ with domains $d(\xi_{1})=X$
and $d(\xi_{2})=Y$ respectively, for each configuration $\conf{x}\in c_{{\phi}^{\downarrow X}}$,
we have that each completion of $\conf{x}$ to $\phi^{\downarrow Y}$
is a  solution of $\phi.$ That is, whenever  
\[
CO(c_{\phi^{\downarrow X}},\phi^{\downarrow Y})\subseteq c_{\phi}.
\]
\end{defn}
\begin{cor}
\label{cor:Total-Projective-Completability}If projective completability
holds on $(\Phi,U),$ $\mathcal{W},$ and $\Xi,$ then for each valuation
$\phi\in\Xi$ such that $\phi=\xi_{1}\times\xi_{2}$ with domains
$d(\xi_{1})=X$ and $d(\xi_{2})=Y$ 
\[
CO(c_{\phi^{\downarrow X}},\phi^{\downarrow Y})=c_{\phi}.
\]
\end{cor}
\begin{proof}
By definition of projective completability we have that $CO(c_{\phi^{\downarrow X}},\phi^{\downarrow Y})\subseteq c_{\phi}.$ 

It only remains to prove that $c_{\phi}\subseteq CO(c_{\phi^{\downarrow X}},\phi^{\downarrow Y}).$

Now, for any $\mathbf{s}\in c_{\phi}$ and by applying equation \ref{eq:FCES-condition}
to $c_{\phi}=W_{\phi}^{\emptyset}(\diamond)$ we have that $\mathbf{s}=\langle\mathbf{x},\mathbf{z}\rangle$
where $\mathbf{x}\in c_{\phi^{\downarrow X}}\mbox{ and }\mathbf{z}\in W_{\phi}^{X}(\mathbf{x}).$
Since $c_{\phi^{\downarrow Y}}=c_{\phi}^{\downarrow Y}$ we get $\langle\mathbf{x}^{\downarrow X\cap Y},\mathbf{z}\rangle\in c_{\phi^{\downarrow Y}}.$
Note that $\mathbf{x}\in c_{\phi^{\downarrow X}}$ implies $\mathbf{x}^{\downarrow X\cap Y}\in c_{\phi^{\downarrow X\cap Y}}$
, and applying equation \ref{eq:FCES-condition} to $c_{\phi^{\downarrow Y}}=W_{\phi^{\downarrow Y}}^{\emptyset}(\diamond),$
we get
\[
c_{\phi^{\downarrow Y}}=\{\langle\mathbf{t},\mathbf{z}\rangle|\mathbf{t}\in c_{\phi^{\downarrow X\cap Y}}\mbox{ and }\mathbf{z}\in W_{\phi^{\downarrow Y}}^{X\cap Y}(\mathbf{t})\}
\]
 we can conclude that $\mathbf{z}\in W_{\phi^{\downarrow Y}}^{X\cap Y}(\mathbf{x}^{\downarrow X\cap Y})$,
and hence that $\mathbf{s}\in CO(c_{\phi^{\downarrow X}},\phi^{\downarrow Y})$.
\end{proof}

\begin{defn}[Piecewise completability]
 We say that \emph{piecewise completability (on products)} holds
on $(\Phi,U),$ $\mathcal{W},$ and $\Xi$ if for each valuation $\phi\in\Xi,$
such that $\phi=\xi_{1}\times\xi_{2}$, with domains $X$ and $Y$
respectively, and each $\mathbf{x}\in c_{\phi^{\downarrow X}}$ ,
each completion of $\mathbf{x}$ to $\xi_{2}$ is a  solution of $\phi$
or equivalently 
\[
CO(c_{\phi^{\downarrow X}},\xi_{2})\subseteq c_{\phi}.
\]

We say that piecewise completability is \emph{guaranteed non-empty}
if for each $A$ such that $\emptyset\neq A\subseteq c_{\phi^{\downarrow X}}$,
we have that $CO(A,\xi_{2})\neq\emptyset$. Furthermore we say that
piecewise completability is \emph{total }if any  solution can be obtained
by piecewise completion, that is if $CO(c_{\phi^{\downarrow X}},\xi_{2})=c_{\phi}.$

\begin{figure}
\centering{}\subfloat[Projective completion]{\usetikzlibrary{arrows,shapes}
\tikzstyle{vertex}=[solid, ellipse, black, draw, minimum width=1cm, minimum height=1cm]
\tikzstyle{edge} = [solid,black,draw,line width=1pt,-]
\tikzstyle{smallvaluation}=[solid, black, draw, minimum width=1.5cm, minimum height=0.5cm]
\begin{tikzpicture}[scale=1, auto,swap]
	
    \node (l2) at (6,6.9) {$Y$}; 
 \draw[<->,blue] (5.25,6.7) - -  (6.75,6.7);
\node (l) at (5,6.7) {$X$};
 \draw[<->,blue] (4.25,6.5) - -  (5.75,6.5);
	\node[smallvaluation] (xi1) at (5,6) {$\xi_1$};
    \node[smallvaluation] (xi2) at (6,5.3) {$\xi_2$};  
	\node[smallvaluation,minimum width=2.5cm] (xi2) at (5.5,4.5) {$\phi = \xi_1\times \xi_2$}; 
\draw[->,red,dashed](4.25,4.2)--(5.25,3.8);
    \draw[->,red,dashed](6.75,4.2)--(6.75,3.8);
\node[smallvaluation] (py) at (6,3.5) {$\phi^{\downarrow Y}$}; 
    \draw[->,blue,dashed](4.25,4.2)--(4.25,3);
    \draw[->,blue,dashed](6.75,4.2)--(5.75,3);
	\node[smallvaluation] (px) at (5,2.7) {$\phi^{\downarrow X}$};
	\node (x) at(5,1.5) {$\mathbf{x}$};
    \node[scale=0.6] (sol) at (5,2) {\ \ Take a solution};
     \draw[black,->] (px)--(sol);
     \draw[black,->] (sol)--(x);
\node[scale=0.6] (cpl) at (6,1.5) {Complete it};
 \draw[black,->] (x)--(cpl);
 \draw[black,->] (py)--(cpl);
\node[] (cplsol) at (6,0.8) {$\langle \mathbf{x},\mathbf{y} \rangle$};
 \draw[black,->] (cpl)--(cplsol);
\end{tikzpicture}

}\qquad{}\qquad{}\qquad{}\qquad{}\subfloat[Piecewise completion]{\usetikzlibrary{arrows,shapes}
\tikzstyle{vertex}=[solid, ellipse, black, draw, minimum width=1cm, minimum height=1cm]
\tikzstyle{edge} = [solid,black,draw,line width=1pt,-]
\tikzstyle{smallvaluation}=[solid, black, draw, minimum width=1.5cm, minimum height=0.5cm]
\begin{tikzpicture}[scale=1, auto,swap]
	
    \node (l2) at (6,6.9) {$Y$}; 
 \draw[<->,blue] (5.25,6.7) - -  (6.75,6.7);
\node (l) at (5,6.7) {$X$};
 \draw[<->,blue] (4.25,6.5) - -  (5.75,6.5);
	\node[smallvaluation] (xi1) at (5,6) {$\xi_1$};
    \node[smallvaluation] (xi2) at (6,5.3) {$\xi_2$};  
	\node[smallvaluation,minimum width=2.5cm] (phi) at (5.5,4.5) {$\phi = \xi_1\times \xi_2$}; 
    \draw[->,blue,dashed](4.25,4.2)--(4.25,3);
    \draw[->,blue,dashed](6.75,4.2)--(5.75,3);
	\node[smallvaluation] (px) at (5,2.7) {$\phi^{\downarrow X}$};
	\node (x) at(5,1.5) {$\mathbf{x}$};
    \node[scale=0.6] (sol) at (5,2) {\ \ Take a solution};
     \draw[black,->] (px)--(sol);
     \draw[black,->] (sol)--(x);
\node[scale=0.6] (cpl) at (6,1.5) {Complete it};
 \draw[black,->] (x)--(cpl);
 \draw[black,->] (xi2) edge[out=0,in=0,looseness=0.6] (cpl);
\node[] (cplsol) at (6,0.8) {$\langle \mathbf{x},\mathbf{y} \rangle$};
 \draw[black,->	] (cpl)  edge (cplsol);
\end{tikzpicture}

}\protect\caption{Process of building a solution by completion}
\end{figure}
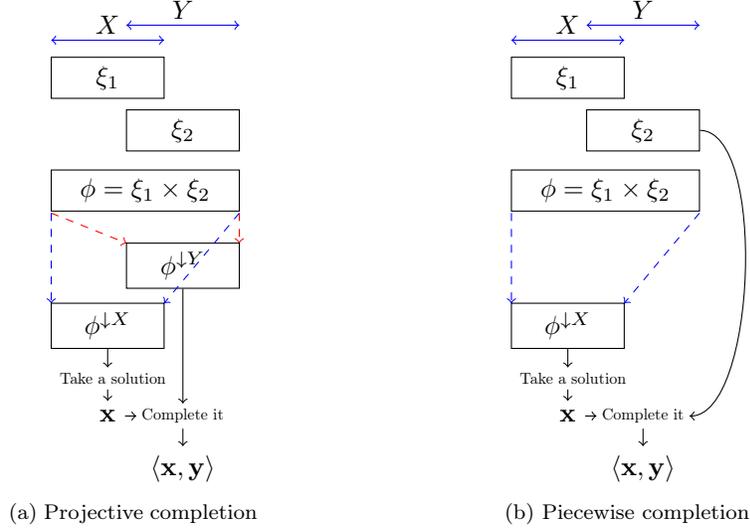

\end{defn}

\subsubsection{Classifying extension systems based on projective and piecewise completability}

In this section we investigate the relationship between piecewise
and projective completability. Since the proofs for these results
rely on valuation algebras on semirings, we only formulate the results
here, leaving the proof to the appendix. \begin{restatable}{prop}{completabilities}
\label{prop:completabilities}
There are valuation algebras and extension system satisfying:
\begin{enumerate}
\item neither projective nor piecewise completability, 
\item projective completability but not piecewise completability,
\item piecewise completability but not projective completability,
\item both piecewise and projective completability. 
\end{enumerate}
\end{restatable}
\begin{proof}
See appendix \ref{app:Piecewise-and-projective}
\end{proof}

\subsection{Necessary and sufficient condition for \EGPName}

\label{sub:EGPSufficientAndNecessaryCondition}

We start by seeing that the projective completability properties,
required only for products of two valuations, can be extended to larger
products by virtue of Lemma \ref{lem:Decomposition}, as long as the
conditions imposed by the lemma hold.
\begin{lem}
\label{lem:ExtendendSolutionsFromIntersection-2} Assume projective
completability holds on $(\Phi,D),$ $\mathcal{W},$ and $\Xi$. Then,
given $\phi=\phi_{1}\times\dots\times\phi_{n},$ and any rooted covering
join tree $\mathcal{T}=(V,E,\lambda,U)$ for that factorization, for
any node $i$ of $V,$ any domain $X\subseteq\lambda^{nde}(i),$ such
that $s_{i}\subseteq X,$ we have that $CO(c_{\phi^{\downarrow X}},\phi^{\downarrow\lambda(i)})=c_{\phi^{\downarrow(X\cup\lambda(i))}}.$
\end{lem}
\begin{proof}
We can apply Lemma \ref{lem:Decomposition} to get that $\phi^{\downarrow(X\cup\lambda(i))}=\alpha\times\beta$,
with $d(\alpha)=X$ and $d(\beta)\subseteq\lambda(i).$ Then, we can
apply Corollary \ref{cor:Total-Projective-Completability} (with $\phi^{\downarrow(X\cup\lambda(i))}$
in the place of $\phi$ and $\lambda(i)$ in that of $Y$) getting
$CO(c_{\phi^{\downarrow X}},\phi^{\downarrow\lambda(i)})=c_{\phi^{\downarrow(X\cup\lambda(i))}}.$

\end{proof}
Now, we are ready to establish the sufficient condition for algorithm
\EGPName, which is basically projective completability.
\begin{thm}
Let $(\Phi,U)$ be a valuation algebra, and $\mathcal{W}$ a extension
system. Let $\Xi\subseteq\Phi$ be a subset of valuations projection-closed
and combination-breakable. Let $\phi\in\Xi,$ and $\mathcal{T}$ a
rooted covering join tree for a given factorization $\phi=\phi_{1}\times\dots\times\phi_{n}.$
Let $c$ be the set of configurations assessed by algorithm \EGPName.
If projective completability holds on $\Xi$ then $c=c_{\phi}.$ \end{thm}
\begin{proof}
Take as loop invariant $c=c_{\phi^{\downarrow\bigcup_{i\in Visited}\lambda(i)}},$
where $Visited$ is the set of nodes of the join tree that have been
visited by the loop up to some point. At the beginning of the first
iteration the invariant is satisfied, since $c=\{\diamond\}=c_{\phi^{\downarrow\emptyset}}.$
For the update of $c$ that is made at each interation, the conditions
of Lemma \ref{lem:ExtendendSolutionsFromIntersection-2} are satisfied,
and hence, the lemma guarantees that if the invariant is true at the
beginning of an iteration, it is true at the end. When the last iteration
finishes, we have visited all the nodes and since $d(\phi)=\cup_{i\in V}\lambda(i)$
by Lemma \ref{lem:LambdasInDomain}, we have that $c=c_{\phi}.$ 
\end{proof}
In the first counterexample provided in section \ref{sec:GenericSolutions}
we had $CO(c_{\phi^{\downarrow X}},\phi^{\downarrow Y})=\Omega_{X\cup Y}$
whereas $c_{\phi}=\{(x\mapsto0,y\mapsto0),(x\mapsto1,y\mapsto1)\}.$
Thus, projective completability does not hold in that counterexample
and hence we do not have any guarantee that the algorithm will work.
In the second counterexample as the extension system is derived from
this one, projective completability does not hold either. Therefore
the misbehaviour of both counterexamples is correctly covered by the
new result.

In the next theorem we establish that projective completability is
also a necessary condition, in the sense that, if for any product
valuation, the algorithm is guaranteed to find a subset of its solutions,
then projective completability must hold.
\begin{thm}
Let $(\Phi,U)$ be a valuation algebra, and $\mathcal{W}$ a extension
system. Let $\Xi\subseteq\Phi$ be a subset of valuations projection-closed
and combination-breakable. If for each valuation $\phi\in\Xi$ which
factorizes as $\phi=\xi_{1}\times\xi_{2}$ and for each rooted covering
join tree $\mathcal{T}$, \EGPName\  assesses $c$ such that $c\subseteq c_{\phi},$
then projective completability holds on $\Xi$.\end{thm}
\begin{proof}
Assume that \EGPName\ \textit{ }\textit{\emph{always assesses}} a
subset of $c_{\phi}.$ For any $\phi\in\Xi,$ $\phi=\xi_{1}\times\xi_{2},$
with domains $X$ and $Y$ respectively, we define a covering join
tree with two nodes: $v_{1},$ with label $\lambda(v_{1})=X$ covering
$\xi_{1},$ and its single child $v_{2},$ with label $\lambda(v_{2})=Y$
covering $\xi_{2}.$ We can run \EGPName\  on $\mathcal{T}$, assessing
$c$ which by our assumption will be a subset of $c_{\phi}.$ By manual
expansion of the expressions in the algorithm, we see that, for this
small tree the solution set assessed is $c=CO(CO(\{\diamond\},\phi^{\downarrow X}),\phi^{\downarrow Y})$.
Now, since $c_{\phi^{\downarrow X}}=CO(\{\diamond\},\phi^{\downarrow X})$,
we have that $c=CO(c_{\phi^{\downarrow X}},\phi^{\downarrow Y})$
and since we had that $c\subseteq c_{\phi},$ we have that $CO(c_{\phi^{\downarrow X}},\phi^{\downarrow Y})\subseteq c_{\phi}.$
Since this holds for any $\phi\in\Xi,$ $\phi=\xi_{1}\times\xi_{2},$
projective completability must hold.
\end{proof}
As noticed by the counterexamples provided in section \ref{sec:GenericSolutions}
the necessary and sufficient conditions for the \EGPName\textit{
}algorithms were not correctly understood in the former literature.
We have provided a characterization of the subsets of a valuation
algebra where the \EGPName\  algorithm works by means of identifying
a sufficient and necessary condition, namely projective completability.

\subsection{Necessary and sufficient condition for \ESName}

\label{sec:ESNecessaryAndSufficient}

As we did in the previous section, we start by seeing that piecewise
completability properties, required only for products of two valuations,
can be extended to larger products by virtue of Lemma \ref{lem:Decomposition},
as long as the conditions imposed by the lemma hold.
\begin{lem}
\label{lem:PiecewiseSolutionsFromIntersection2}Let $(\Phi,U)$ be
a valuation algebra, and $\mathcal{W}$ a extension system. Let $\Xi\subseteq\Phi$
be a subset of valuations projection-closed and combination-breakable.
Then, for any $\phi\in\Xi,$ any rooted covering join tree $\mathcal{T}$
for a given factorization $\phi=\phi_{1}\times\dots\times\phi_{n},$
any node $i$ of $\mathcal{T},$ any domain $X\subseteq\lambda^{nde}(i),$
such that $s_{i}\subseteq X,$ and any set $A\subseteq c_{\phi^{\downarrow X}}$,
we define $\beta=\left(\psi_{i}\times\prod_{j\in de(i)}\psi_{j}\right)^{\downarrow\lambda(i)}$
and we have that 
\begin{enumerate}
\item If piecewise completability holds, then $CO(A,\beta)\subseteq c_{\phi^{\downarrow(X\cup\lambda(i))}},$ 
\item If piecewise completability is guaranteed non-empty, then whenever
$A\neq\emptyset$ we have that $CO(A,\beta)\neq\emptyset.$ 
\item If piecewise completability is complete we have that $CO(c_{\phi^{\downarrow X}},\beta)=c_{\phi^{\downarrow(X\cup\lambda(i))}}.$
\end{enumerate}
\end{lem}
\begin{proof}
We can apply Lemma \ref{lem:Decomposition} to get that $\phi^{\downarrow(X\cup\lambda(i))}=\alpha\times\beta$,
with $d(\alpha)=X$, $d(\beta)=\lambda(i),$ and $\beta=\left(\psi_{i}\times\prod_{j\in de(i)}\psi_{j}\right)^{\downarrow\lambda(i)}.$
The conditions to apply piecewise completability hold (with $\phi^{\downarrow(X\cup\lambda(i))}$
in place of $\phi$ and $\lambda(i)$ in that of $Y$) getting that
$CO(A,\beta)\subseteq c_{\phi^{\downarrow(X\cup\lambda(i))}}.$ The
second and third statements can be proven the same way.
\end{proof}
Following what we did with \EGPName, now we are ready to establish
the sufficient contitions for the \ESName\  algorithm, namely piecewise
completability in its different flavors.
\begin{thm}
\label{thm:Extend-to-Subtree-Sufficient}Let $(\Phi,U)$ be a valuation
algebra, and $\mathcal{W}$ a extension system. Let $\Xi\subseteq\Phi$
be a subset of valuations projection-closed and combination-breakable.
Let $\phi\in\Xi,$ and $\mathcal{T}$ be a rooted covering join tree
for a given factorization $\phi=\phi_{1}\times\dots\times\phi_{n}.$
Let $c$ be the set of configurations assessed by algorithm \ESName.
We have that
\begin{enumerate}
\item If piecewise completability holds on $\Xi$, then $c$ is a subset
of $c_{\phi}.$ 
\item If piecewise completability is guaranteed non-empty on $\Xi$ then
also $c\neq\emptyset$. 
\item If piecewise completability is total on $\Xi$, then $c=c_{\phi}.$
\end{enumerate}
\end{thm}
\begin{proof}
To prove statement 1, take as loop invariant $c\subseteq c_{\phi^{\downarrow\bigcup_{i\in Visited}\lambda(i)}},$
where $Visited$ is the set of nodes of the join tree that have been
visited by the loop up to some point. At the beginning of the first
iteration the invariant is satisfied, since $c=\{\diamond\}\subseteq c_{\phi^{\downarrow\emptyset}}.$
For the update of $c$ that is made at each interation, the conditions
of Lemma \ref{lem:PiecewiseSolutionsFromIntersection2} are satisfied,
and hence, the lemma guarantees that if the invariant is true at the
beginning of an iteration, it is true at the end. When the last iteration
finishes, we have visited all the nodes and since Lemma \ref{lem:LambdasInDomain}
shows that $d(\phi)=\cup_{i\in V}\lambda(i),$ we have that $c\subseteq c_{\phi}.$
Statements 2 and 3 can be proven the same way.
\end{proof}
Again, piecewise completability is not only a sufficient condition,
but also necessary if the \ESName\ \noun{ }algorithm works in a consistent
manner, as proven by the following theorem.
\begin{thm}
\label{thm:Extend-to-Subtree-Necessary}Let $(\Phi,U)$ be a valuation
algebra, and $\mathcal{W}$ a extension system. Let $\Xi\subseteq\Phi$
be a subset of valuations projection-closed and combination-breakable.
If for each valuation $\phi\in\Xi,$ $\phi=\xi_{1}\times\xi_{2}$
and for each rooted covering join tree $\mathcal{T}$, \ESName\ 
assesses $c$ such that 
\begin{enumerate}
\item $c\subseteq c_{\phi},$ then piecewise completability holds on $\Xi$.
\item $\emptyset\neq c\subseteq c_{\phi},$ then piecewise completability
is guaranteed non-empty on $\Xi$.
\item $c=c_{\phi},$ then piecewise completability is total on $\Xi$.
\end{enumerate}
\end{thm}
\begin{proof}
We start proving statement 1. Assume that \ESName\ \textit{ }\textit{\emph{always
assesses}} a subset of $c_{\phi}.$ Given $\phi\in\Xi,$ $\phi=\xi_{1}\times\xi_{2},$
with domains $X$ and $Y$ respectivelty, we define a join tree with
two nodes: $v_{1},$ with label $\lambda(v_{1})=X$ covering $\xi_{1},$
and its single child $v_{2},$ with label $\lambda(v_{2})=Y$ covering
$\xi_{2}.$ We can run \noun{Extend-To-Subtree-Projections} on $\mathcal{T}$,
getting $c\subseteq c_{\phi}.$ However in this particular case we
can see that $c=CO(CO(\{\diamond\},\phi^{\downarrow X}),\xi_{2})$.
Now, since $c_{\phi^{\downarrow X}}=CO(\{\diamond\},\phi^{\downarrow X})$,
we have that $c=CO(c_{\phi^{\downarrow X}},\xi_{2})$. Since the algorithm
is guaranteed to return $c\subseteq c_{\phi},$ piecewise completability
must hold. Statements 2 and 3 can be proven the same way.
\end{proof}

\subsection{Sufficient conditions for \SESName}

Finally, we show that the \SESName\  algorithm can be applied if
guaranteed non-empty piecewise completability holds. 
\begin{thm}
\label{thm:PiecewiseSETSP}Let $(\Phi,U)$ be a valuation algebra,
and $\mathcal{W}$ a extension system. Let $\Xi\subseteq\Phi$ be
a subset of valuations projection-closed and combination-breakable.
Let $\phi\in\Xi,$ and $\mathcal{T}$ be a rooted covering join tree
for a given factorization $\phi=\phi_{1}\times\dots\times\phi_{n}.$
If guaranteed non-empty piecewise completability holds on $\Xi$,
then\noun{ }\SESName\  assesses a configuration $\mathbf{x}$ which
is a solution to $\phi.$ \end{thm}
\begin{proof}
Take as loop invariant $\mathbf{x}\in c_{\phi^{\downarrow\bigcup_{i\in Visited}\lambda(i)}},$
where $Visited$ is the set of nodes of the join tree that have been
visited by the loop up to some point. At the beginning of the first
iteration the invariant is satisfied, since $\diamond\in c_{\phi^{\downarrow\emptyset}}.$
The update of $\mathbf{x}$ that made at each interation, is possible
because the conditions of Lemma \ref{lem:PiecewiseSolutionsFromIntersection2}
(including guaranteed non-emptyness) are satisfied, and hence, there
is a completion that we can select, store in $\mathbf{x},$ and it
is guaranteed to maintain the invariant. When the last iteration finishes,
we have visited all the nodes and since $d(\phi)=\cup_{i\in V}\lambda(i)$
due to Lemma \ref{lem:LambdasInDomain}, we have that $\mathbf{x}\in c_{\phi}.$
\end{proof}
In the last three sections we have characterized under which circumstances
can we apply each algorithm. In the next section we compare with the
sufficient conditions provided by Pouly and Kohlas.

\subsection{Alternative proofs for the PK results}

As we argued before, Pouly and Kohlas stated that \EGPName\  always
assessed $c_{\phi},$ and we disproved by means of counterexamples.
However, we have proven that projective completability is a sufficient
and necessary condition for the algorithm. They did also provide sufficient
conditions for the algorithms \ESName\  and \SESName. We will see
that, although the proofs relied on a disproved theorem, the results
provided were correct. We do that by proving that the sufficient conditions
established by them and described in section \ref{sec:PK-Sufficient-Conditions}
imply our sufficient conditions.

As can be seen in Table \ref{tab:Impact}, CPK1 and CPK2 were proposed
as sufficient condition for \ESName\ \noun{ }to assess some solutions
and for \SESName\  to assess a solution. The following lemma allows
us to use theorems \ref{thm:Extend-to-Subtree-Sufficient} and \ref{thm:PiecewiseSETSP}
to prove that their conditions were indeed sufficient.
\begin{lem}
\label{lem:PK-Implies-Piecewise}Assume that conditions CPK1 and CPK2
hold. Then, guaranteed non-empty piecewise completability holds on
$\Phi$.\end{lem}
\begin{proof}
Take $\xi_{1},\xi_{2}\in\Phi$, with domains $X$ and $Y$ respectively,
and let $\phi=\xi_{1}\times\xi_{2}.$ To prove piecewise completability,we
have to prove that $CO(c_{\phi^{\downarrow X}},\xi_{2})\subseteq c_{\phi}.$
Now by definition $CO(c_{\phi^{\downarrow X}},\xi_{2})=\{(\conf{x},\mathbf{z})\mid\conf{x}\in c_{\phi^{\downarrow X}}\mbox{ and }\conf{z}\in W_{\xi_{2}}^{X\cap Y}(\conf{x}^{\downarrow X\cap Y})\}.$
We can apply CPK2 to get $CO(c_{\phi^{\downarrow X}},\xi_{2})\subseteq\{(\conf{x},\mathbf{z})\mid\conf{x}\in c_{\phi^{\downarrow X}}\mbox{ and }\conf{z}\in W_{\phi}^{X}(\conf{x})\}.$
Now, by the second condition in the definition of extension system
we get $CO(c_{\phi^{\downarrow X}},\xi_{2})\subseteq\{(\conf{x},\mathbf{z})\mid\conf{x}\in c_{\phi^{\downarrow X}}\mbox{ and }\conf{z}\in W_{\phi}^{X}(\conf{x})\}=W_{\phi}^{\emptyset}(\diamond)=c_{\phi}$
and piecewise completability is proven. Now we need to see that non-emptyness
is guaranteed. Take $A$ such that $\emptyset\neq A\subseteq c_{{\phi}^{\downarrow X}}$,
we have to prove that $CO(A,\xi_{2})\neq\emptyset.$ Again by definition
$CO(A,\xi_{2})=\{(\conf{x},\mathbf{z})\mid\conf{x}\in A\mbox{ and }\conf{z}\in W_{\xi_{2}}^{X\cap Y}(\conf{x}^{\downarrow X\cap Y})\}$.
By CPK1 we have that $W_{\xi_{2}}^{X\cap Y}(\conf{x}^{\downarrow X\cap Y})\neq\emptyset$
and so $CO(A,\xi_{2})$ is guaranteed to be non-empty.
\end{proof}
Furthermore, CPK1 and CPK3 were identified as a sufficient condition
for \ESName\  to assess $c_{\phi}.$ The following lemma allows us
to use theorems \ref{thm:Extend-to-Subtree-Sufficient} and \ref{thm:PiecewiseSETSP}
to prove that their conditions were indeed sufficient.
\begin{lem}
\label{lem:PK-Implies-Piecewise-1}Assume that conditions CPK1 and
CPK3 hold. Then, total piecewise completability holds on $\Phi$. \end{lem}
\begin{proof}
Take $\xi_{1},\xi_{2}\in\Phi$, with domains $X$ and $Y$ respectively,
and let $\phi=\xi_{1}\times\xi_{2}.$ To prove piecewise total completability,
we have to prove that $CO(c_{\phi^{\downarrow X}},\xi_{2})=c_{\phi}.$
Now by definition $CO(c_{\phi^{\downarrow X}},\xi_{2})=\{(\conf{x},\mathbf{z})\mid\conf{x}\in c_{\phi^{\downarrow X}}\mbox{ and }\conf{z}\in W_{\xi_{2}}^{X\cap Y}(\conf{x}^{\downarrow X\cap Y})\}.$
We can apply CPK3 to get $CO(c_{\phi^{\downarrow X}},\xi_{2})=\{(\conf{x},\mathbf{z})\mid\conf{x}\in c_{\phi^{\downarrow X}}\mbox{ and }\conf{z}\in W_{\phi}^{X}(\conf{x})\}.$
Now, by the second condition in the definition of extension system
we get $CO(c_{\phi^{\downarrow X}},\xi_{2})=\{(\conf{x},\mathbf{z})\mid\conf{x}\in c_{\phi^{\downarrow X}}\mbox{ and }\conf{z}\in W_{\phi}^{X}(\conf{x})\}=W_{\phi}^{\emptyset}(\diamond)=c_{\phi}$
and total piecewise completability is proven. 
\end{proof}
On the other hand, we point out that in both cases the sufficient
conditions we require, while similar to the ones required by Pouly
and Kohlas are strictly weaker than those. In particular, their conditions
need to hold on any configuration $\mathbf{x}\in\Omega_{X},$ whilst
we only require them to hold for those $\mathbf{x}\in c_{\phi^{\downarrow X}}.$
That is, while they impose conditions on the extension of tuples which
are not solutions, we restrict ourselves to solutions. Table \ref{tab:Impact-1}
summarizes the results in this section, providing the sufficient conditions
for each algorithm, whether the condition has also been proven to
be also necessary and whether the condition we require is weaker than
the one previously required. 

\begin{table}
\begin{centering}
{\small{}}%
\begin{tabular}{|c|>{\centering}m{4cm}|>{\centering}m{1.7cm}|>{\centering}m{1.2cm}|c|}
\hline 
\textbf{\scriptsize{}Algorithm } & \textbf{\scriptsize{}Suff. cond. } & \textbf{\scriptsize{}Nec. cond. } & \textbf{\scriptsize{}Weaker} & \textbf{\scriptsize{}Solutions}\tabularnewline
\hline 
\hline 
\noun{\scriptsize{}\ref{alg:ExtendAll}} & {\scriptsize{}Projective completability} & {\scriptsize{}Yes} & {\scriptsize{}-} & {\scriptsize{}All}\tabularnewline
\hline 
\noun{\scriptsize{}\ref{alg:ExtendOne}} & {\scriptsize{}Guaranteed non-empty piecewise completability} & {\scriptsize{}No} & {\scriptsize{}Yes} & {\scriptsize{}One}\tabularnewline
\hline 
\noun{\scriptsize{}\ref{alg:ExtendSome}} & {\scriptsize{}Guaranteed non-empty piecewise completability} & {\scriptsize{}Yes} & {\scriptsize{}Yes} & {\scriptsize{}Some}\tabularnewline
\hline 
\noun{\scriptsize{}\ref{alg:ExtendSome}} & {\scriptsize{}Total piecewise completability} & {\scriptsize{}Yes} & {\scriptsize{}Yes} & {\scriptsize{}All}\tabularnewline
\hline 
\end{tabular}
\par\end{centering}{\small \par}

\protect\caption{\label{tab:Impact-1}Sufficient and necessary conditions of the generic
algorithms.}
\end{table}

\section{Optimization problems in semiring induced valuation algebras}

\label{sec:Optimization}

Many problems in Artificial Intelligence can be expressed in terms
of a particular type of valuations, namely semiring induced valuation
algebras, that emerge from a mapping from tuples to the values of
a commutative semiring \cite{Bistarelli1997,Bistarelli2004,Kohlas2008,Werner2014}.
Particularly interesting are optimization problems, where the semiring
is selective. We start by reviewing optimization problems and the
result that an extension system can be defined when the semiring is
selective. Then, by means of a counterexample, we show that the sufficient
condition imposed by Pouly and Kohlas for the correctness of\noun{
\EGPName\ } is not correct and propose a sufficient condition and
a necessary condition for projective completability to hold on valuation
algebras imposed by a selective semiring, and thus, for\noun{ \EGPName\ }
to work. Later, for \SESName\  and \ESName\  we provide correct
proofs for the sufficient conditions introduced by Pouly and Kohlas
to solve the single and partial SFP. Finally we show that we can weaken
the sufficient condition proposed by Pouly and Kohlas for \ESName\ 
to solve the complete SFP from strict monotonicity to weak cancellativity.
Furthermore we show that weak cancellativity is also a necessary condition.

\subsection{Optimization problems.}

We start by defining some basic abstract algebra structures needed
to specify the problem and then we formally state the problem, which
is a particular case of the SFP. 
\begin{defn}
A \emph{semiring} is a set $R$ equipped with two binary operations
$+$ and $\cdot$, called addition and multiplication, such that (i)
$+$ is an associative and commutative operation with identity element
$0$, (ii) $\cdot$ is an associative operation with identity element
$1$, (iii) multiplication left and right distributes over addition,
that is $a\cdot(b+c)=(a\cdot b)+(a\text{\ensuremath{\cdot}}c)$ and
$(a+b)\cdot c=(a\cdot c)+(b\cdot c)$, and (iv) multiplication by
0 annihilates $R$, that is $a\cdot0=0\cdot a=0$.

If $\cdot$ is commutative then $(R,+,\cdot)$ is a \emph{commutative
semiring.}\end{defn}
\begin{thm}
Let $\langle U,\Omega\rangle$ be a variable system, and $(R,+,.)$
a commutative semiring\emph{.} A \emph{semiring valuation} $\phi$
with domain $X\subseteq U$ is a function $\phi:\Omega_{X}\rightarrow R.$
The set of all semiring valuations with domain $X$ is noted $\Phi_{X}$,
and $\Phi=\bigcup_{X\subseteq U}\Phi_{X}.$ Now we define $d(\phi)=X$
if $\phi\in\Phi_{X}.$ Furthermore $(\phi\times\psi)(\mathbf{x})=\phi(\mathbf{x}^{\downarrow d(\phi)})\cdot\psi(\mathbf{x}^{\downarrow d(\psi)}).$
And finally $\phi^{\downarrow Y}(\mathbf{y})=\sum_{\mathbf{z}\in\Omega_{d(\phi)-Y}}\phi(\langle\mathbf{y},\mathbf{z}\rangle)$
for $Y\subseteq d(\phi)$. With these operations, $(\Phi,U)$ satisfies
the axioms of a valuation algebra and is called the valuation algebra
induced by $(R,+,.)$ in $\langle U,\Omega\rangle$.\end{thm}
\begin{proof}
See Theorem PK5.2.
\end{proof}
Thus, in the following we are only interested in commutative semirings.
Note that Example \ref{ex:ValAlgBooleanFunctions} is indeed a semiring
induced valuation algebra. 
\begin{defn}
\label{def:SolutionConceptSemiring}For any semiring induced valuation
algebra, the \emph{optimization solution concept} assigns at each
$\phi\in\Phi_{X},$ the set of solutions $c_{\phi}=\{\mathbf{x}\in\Omega_{X}|\phi(\mathbf{x})=\phi^{\downarrow\emptyset}(\diamond)\}.$
Thus we define the\emph{ single (resp. partial, complete) optimization
solution finding problem} as the single\emph{ }(resp. partial, complete)
solution finding problem with this solution concept on the valuation
algebra induced by that semiring\emph{.} 
\end{defn}
The former definition of optimization problem covers several common
optimization formalisms, such as Classical Optimization, Satisfiability,
Maximum Satisfability, Most \& Least Probable Values, Bayesian and
Maximum Likelihood decoding and Linear decoding. Details can be found
in \cite{Pouly2011a}.

\subsection{A extension system for optimization.}

We are interested in determining whether we can use the algorithms
presented in section \ref{sec:Solving-SFP}. The first requirement
for those algorithms was the existence of an extension system, which
we will prove in this section. In order to define an extension system,
we need to impose a condition on the semiring, namely being selective\footnote{Former literature used to require totally ordered idempotent semirings.
As shown in corollary \ref{cor:SelectiveSemiring} in the appendix,
both conditions are equivalent. Thus, we use selective semirings to
simplify the wording.}. 
\begin{defn}
A semiring is $(R,+,\cdot)$ is \emph{selective} if for all $a,b\in R,$
either $a+b=a$ or $a+b=b$.

In a selective semiring, we can define a relation 
\[
a\leq b\Longleftrightarrow a+b=b.
\]
As a consequence of Proposition 3.4.7 in \cite{Gondran2008}, in any
selective semiring $\leq$ is a total order relation. It is immediate
to see that in any selective semiring $a+b=\max\{a,b\}$, where the
maximum is taken with respect to the total order $\leq.$ Note that,
since $0$ is the sum's identity, we have that $0\leq a$ for all
$a.$
\end{defn}

\begin{defn}
\label{def:OptiFCES}Given the valuation algebra induced by a selective
semiring\emph{,} we define the \emph{optimization extension system}
as the family of sets obtained by defining the set of extensions of
$\mathbf{x}$ to $\phi,$ where $\mathbf{x\in\Omega}_{X}$ and $X\subseteq d(\phi)$,
as 
\begin{equation}
W_{\phi}^{X}(\mathbf{x})=\{\mathbf{z}\in\Omega_{d(\phi)-X}|\phi(\langle\mathbf{x},\mathbf{z}\rangle)=\phi^{\downarrow X}(\mathbf{x})\}.\label{eq:optimization-Extension}
\end{equation}

Notice that $W_{\phi}^{\emptyset}(\diamond)$ is equal to $c_{\phi}$
as defined by the optimization solution concept.\end{defn}
\begin{lem}
The optimization extension system satisfies the condition in equation
\ref{eq:FCES-condition} and hence, it is an extension system.\end{lem}
\begin{proof}
We want to prove $W_{\phi}^{X}(\mathbf{x})=\{\langle\conf{y},\conf{z}\rangle|\conf{y}\in W_{\phi^{\downarrow Y}}^{X}(\conf{x})\mbox{ and }\conf{z}\in W_{\phi}^{Y}(\langle\conf{x},\conf{y}\rangle)\}$
for $X\subseteq Y\subseteq d(\phi).$ In order to simplify the notation
take $A=\{\langle\conf{y},\conf{z}\rangle|\conf{y}\in W_{\phi^{\downarrow Y}}^{X}(\conf{x})\mbox{ and }\conf{z}\in W_{\phi}^{Y}(\conf{x},\conf{y})\}.$

It follows from equation \ref{eq:optimization-Extension} that 
\begin{eqnarray}
A & = & \{\langle\mathbf{y},\mathbf{z}\rangle\mid\phi^{\downarrow X}(\mathbf{x})=\phi^{\downarrow Y}(\langle\mathbf{x},\mathbf{y}\rangle)\mbox{ and \ensuremath{\phi^{\downarrow Y}}(\ensuremath{\langle\mathbf{x}},\ensuremath{\mathbf{y}\rangle}=}\phi(\langle\langle\mathbf{x},\mathbf{y}\rangle,\mathbf{z}\rangle)\}\nonumber \\
 & = & \{\langle\mathbf{y},\mathbf{z}\rangle\mid\phi^{\downarrow X}(\mathbf{x})=\phi^{\downarrow Y}(\langle\mathbf{x},\mathbf{y}\rangle)=\phi(\langle\langle\mathbf{x},\mathbf{y}\rangle,\mathbf{z}\rangle)\}.\label{eq:FCES1}
\end{eqnarray}

For any $\mathbf{t}\in A$ we have that $\mathbf{t}=\langle\mathbf{y},\mathbf{z}\rangle$
with $\mathbf{y}\in\Omega_{Y-X}$ and $\mathbf{z}\in\Omega_{d(\phi)-(X\cup Y)}$,
therefore $\mathbf{t}\in\Omega_{d(\phi)-X}$. Hence, the domain of
the tuples in $A$ and $W_{\phi}^{X}(\mathbf{x})$ are actually the
same.
\begin{itemize}
\item We prove that $A\subseteq W_{\phi}^{X}(\mathbf{x}).$ Take $\mathbf{t}\in A$.
We have that $\mathbf{t}=\langle\mathbf{y},\mathbf{z}\rangle,$ and
by equation \ref{eq:FCES1} that $\phi^{\downarrow X}(\mathbf{x})=\phi(\langle\langle\mathbf{x},\mathbf{y}\rangle,\mathbf{z}\rangle).$
Now, by the associativity of the concatenation of tuples we have that
$\phi^{\downarrow X}(\mathbf{x})=\phi(\langle\langle\mathbf{x},\langle\mathbf{y},\mathbf{z}\rangle\rangle)=\phi(\langle\mathbf{x},\mathbf{t}\rangle),$
proving that $\mathbf{t}\in W_{\phi}^{X}(\mathbf{x}).$
\item We prove that $W_{\phi}^{X}(\mathbf{x})\subseteq A.$ Take $\mathbf{t}\in W_{\phi}^{X}(\mathbf{x})$.
From the definition of $W_{\phi}^{X}(\mathbf{x}),$ we have 
\[
\begin{array}{c}
\phi(\langle\mathbf{x},\mathbf{t}\rangle)=\phi^{\downarrow X}(\mathbf{x})=\sum_{\mathbf{z}\in\Omega_{d(\phi)-X}}\phi(\langle\mathbf{x},\mathbf{z}\rangle)\mbox{, and }\\
\phi^{\downarrow Y}(\langle\mathbf{x},\mathbf{t}^{\downarrow Y-X}\rangle)=\sum_{\mathbf{z'}\in\Omega_{d(\phi)-Y}}\phi(\langle\langle\mathbf{x},\mathbf{t}^{\downarrow Y-X}\rangle,\mathbf{z'}\rangle).
\end{array}
\]
Since our semiring is selective, we can apply that $a+b=\max\{a,b\}$,
to obtain 
\[
\begin{array}{c}
\begin{array}{ccc}
\phi(\langle\mathbf{x},\mathbf{t}\rangle) & = & \max_{\mathbf{z}\in\Omega_{d(\phi)-X}}\phi(\langle\mathbf{x},\mathbf{z}\rangle)\mbox{, and }\\
\phi^{\downarrow Y}(\langle\mathbf{x},\mathbf{t}^{\downarrow Y-X}\rangle) & = & \max_{\mathbf{z}\in\Omega_{d(\phi)-Y}}\phi(\langle\langle\mathbf{x},\mathbf{t}^{\downarrow Y-X}\rangle,\mathbf{z}\rangle)\\
 & = & \max_{\mathbf{z}\in\Omega_{d(\phi)-Y}}\phi(\langle\mathbf{x},\langle\mathbf{t}^{\downarrow Y-X},\mathbf{z}\rangle\rangle).
\end{array}\end{array}
\]
Then $\phi(\langle\mathbf{x},\mathbf{t}\rangle)\leq\max_{\mathbf{z}\in\Omega_{d(\phi)-Y}}\phi(\langle\mathbf{x},\langle\mathbf{t}^{\downarrow Y-X},\mathbf{z}\rangle\rangle)=\phi^{\downarrow Y}(\langle\mathbf{x},\mathbf{t}^{\downarrow Y-X}\rangle).$
On the other hand 
\begin{eqnarray*}
\phi^{\downarrow Y}(\langle\mathbf{x},\mathbf{t}^{\downarrow Y-X}\rangle) & = & \max_{\mathbf{z}\in\Omega_{d(\phi)-Y}}\phi(\langle\mathbf{x},\langle\mathbf{t}^{\downarrow Y-X},\mathbf{z}\rangle\rangle)\\
 & \leq & \max_{\mathbf{z}\in\Omega_{d(\phi)-X}}\phi(\langle\mathbf{x},\mathbf{z}\rangle)=\phi^{\downarrow X}(\mathbf{x})=\phi(\langle\mathbf{x},\mathbf{t}\rangle),
\end{eqnarray*}
 which, since the order is total, proves $\phi^{\downarrow Y}(\langle\mathbf{x},\mathbf{t}^{\downarrow Y-X}\rangle)=\phi(\langle\mathbf{x},\mathbf{t}\rangle)=\phi^{\downarrow X}(\mathbf{x})$
and hence $\mathbf{t}\in A.$ 
\end{itemize}
\end{proof}
The extension system defined in Example \ref{ex:FCESValAlgBooleanFunctions}
is an optimization extension system.

\subsection{Necessary and sufficient conditions for \EGPName\ on optimization
problems}

Pouly and Kohlas claimed that \noun{\EGPName\ } solves the complete
optimization SFP on any valuation algebra induced by a commutative
selective semiring. The following counterexample shows that this is
not correct. 

\begin{counter}\label{cnt:NotSquareOrdered}

There are valuation algebras induced by selective semirings where
\EGPName\  does not solve the optimization complete SFP. 

\end{counter}
\begin{proof}
We start by defining a commutative selective semiring over the subset
of integers $R=\{0,1,2,3\}$. The sum is defined as the maximum of
the two integers, that is $a+b=$$\max(a,b).$ The product is defined
by the following table

\begin{center}
\begin{tabular}{|c|c|c|c|c|}
\hline 
$\cdot$ & \textbf{0} & \textbf{1} & \textbf{2} & \textbf{3}\tabularnewline
\hline 
\hline 
\textbf{0} & 0 & 0 & 0 & 0\tabularnewline
\hline 
\textbf{1} & 0 & 1 & 2 & 3\tabularnewline
\hline 
\textbf{2} & 0 & 2 & 2 & 3\tabularnewline
\hline 
\textbf{3} & 0 & 3 & 3 & 3\tabularnewline
\hline 
\end{tabular}
\par\end{center}

It is easy to check directly that $(R,\max,\cdot)$ is a commutative
selective semiring.

Now, we take two boolean variables $x$ and $y$ and define the valuations:

\begin{eqnarray*}
\phi_{1}(x)=\begin{cases}
2 & \mbox{, if }x=0\\
3 & \mbox{, if }x=1
\end{cases} & \ \ \ \ \ \  & \phi_{2}(y)=\begin{cases}
2 & \mbox{, if }y=0\\
3 & \mbox{, if }y=1.
\end{cases}
\end{eqnarray*}

The product $\phi=\phi_{1}\times\phi_{2},$ is $\phi(x,y)=\begin{cases}
2 & \mbox{, if }x=y=0\\
3 & \mbox{, otherwise.}
\end{cases}$

The solutions of $\phi$ are the assignments $\{\{x\mapsto0,y\mapsto1\},\{x\mapsto1,y\mapsto0\},\{x\mapsto1,y\mapsto1\}\}.$
However, the result of running algorithm \EGPName\  will include
the assignment $\{x\mapsto0,y\mapsto0\}$ which is not a solution.

\end{proof}
The need to identify a sufficient condition where the algorithm solves
the complete optimization SFP arises as a consequence of the counterexample.
Note that we have already identified a sufficient and necessary condition
in section \ref{sub:EGPSufficientAndNecessaryCondition}, namely projective
completability. What we would like to see is whether we can transform
this condition into a condition of the semiring. We will start by
defining two conditions on a semiring and seeing that for commutative
semirings, one implies the other. Then, we will prove that the stronger
condition is sufficient and that the weaker condition is necessary.
\begin{defn}
A selective semiring $(R,+,\cdot)$ is \emph{square multiplicatively
cancellative on image} if for each $a,b\in Im(\cdot),$ $a\neq0,$
having $a\cdot a=b\cdot a$ implies $a=b.$

A selective semiring $(R,+,\cdot)$ is \emph{square ordered} if for
each $a,b\in R,$ having $a\cdot a=b\cdot a$ implies that $b\cdot b\geq a\cdot a.$ \end{defn}
\begin{prop}
If a selective semiring is commutative and square multiplicatively
cancellative on image then it is square ordered.\end{prop}
\begin{proof}
By reductio ad absurdum. Let's assume that $(R,+,\cdot)$ is not square
ordered. This means that there are $a,b\in R$ such that $a\cdot a=b\cdot a$
and $b\cdot b<a\cdot a.$ Now, take $c=a\cdot a$ and $d=b\cdot b$.
These are two elements in $Im(\cdot)$ and since $b\cdot b<a\cdot a,$
we have that $c=a\cdot a>b\cdot b\geq0.$ Also, notice that $c\cdot c=a\cdot a\cdot a\cdot a$.
Since $a\cdot a=b\cdot a$, we get that $c\cdot c=b\cdot a\cdot b\cdot a$
and since the semiring is commutative, we have that $c\cdot c=b\cdot b\cdot a\cdot a=d\cdot c$.
So, applying that R is square multiplicatively cancellative on image,
we get that $c=d$, that is $a\cdot a=b\cdot b,$ which contradicts
that $b\cdot b<a\cdot a.$ \end{proof}
\begin{thm}
Let $(\Phi,U)$ be a valuation algebra induced by a selective commutative
semiring $(R,+,\cdot)$. If $(R,+,\cdot)$ is square multiplicatively
cancellative on image, then projective extensibility holds on $\Phi.$ \end{thm}
\begin{proof}
We start by assuming that the semiring is square multiplicatively
cancellative on image and we see that projective completability holds.
We have to prove that for any valuation $\phi=\phi_{1}\times\phi_{2}$
with $d(\phi_{1})=X$ and $d(\phi_{2})=Y$ we have that $CO(c_{\phi^{\downarrow X}},\phi^{\downarrow Y})\subseteq c_{\phi}$.
To make the notation simpler the value of the solution, namely $\phi^{\downarrow\emptyset}(\diamond),$
will be written as $M.$ 

If $M=\phi^{\downarrow\emptyset}(\diamond)=0$ then $\phi(\mathbf{t})=0$
for all $\mathbf{t}\in\Omega_{d(\phi)}$, since $\phi^{\downarrow\emptyset}(\diamond)=\sum_{\mathbf{t}\in\Omega_{d(\phi)}}\phi(\mathbf{t})=\max_{\mathbf{t}\in\Omega_{d(\phi)}}\phi(\mathbf{t})$,
and $0$ is the minimal element. Hence, all the configurations are
solutions, and $CO(c_{\phi^{\downarrow X}},\phi^{\downarrow Y})\subseteq\Omega_{d(\phi)}=c_{\phi},$
and projective completability is guaranteed. 

So we only need study the case when $M\neq0.$ In that case, take
$\langle\mathbf{x},\mathbf{z}\rangle\in CO(c_{\phi^{\downarrow X}},\phi^{\downarrow Y}).$
By definition of completion we have that $\mathbf{x}\in c_{\phi^{\downarrow X}}$
and $\mathbf{z}\in W_{\phi^{\downarrow Y}}^{X\cap Y}(\mathbf{x}^{\downarrow X\cap Y}).$
By definition of $c_{\phi^{\downarrow X}}$, for any $\mathbf{x}\in c_{\phi^{\downarrow X}}$
we have $\phi^{\downarrow X}(\mathbf{x})=\phi^{\downarrow\emptyset}(\diamond)=M,$
and hence we have that $\phi^{\downarrow X\cap Y}(\mathbf{x}^{\downarrow X\cap Y})=M.$

Furthermore, if $\mathbf{z}\in W_{\phi^{\downarrow Y}}^{X\cap Y}(\mathbf{x}^{\downarrow X\cap Y}),$
by equation \ref{eq:optimization-Extension}, we have $\phi^{\downarrow Y}(\mathbf{\langle x}^{\downarrow X\cap Y},\mathbf{z\rangle})=\phi^{\downarrow X\cap Y}(\mathbf{x}^{\downarrow X\cap Y}),$
and since from the previous paragraph we have that $\phi^{\downarrow X\cap Y}(\mathbf{x}^{\downarrow X\cap Y})=M,$
we can conclude that $\phi^{\downarrow Y}(\mathbf{\langle x}^{\downarrow X\cap Y},\mathbf{z\rangle})=M.$
In order to finish the proof we need to see that $\langle\mathbf{x},\mathbf{z}\rangle\in c_{\phi}.$ 

By using the combination axiom we have 
\[
M=\phi^{\downarrow\emptyset}(\diamond)=\phi^{\downarrow X}(\mathbf{x})=(\phi_{1}\times\phi_{2})^{\downarrow X}(\mathbf{x})=\phi_{1}(\mathbf{x})\cdot\phi_{2}^{\downarrow X\cap Y}(\mathbf{x}{}^{\downarrow X\cap Y})
\]
 and 
\begin{eqnarray*}
M & = & \phi^{\downarrow\emptyset}(\diamond)=\phi^{\downarrow Y}(\langle\mathbf{x}{}^{\downarrow X\cap Y},\mathbf{z\rangle})=(\phi_{1}\times\phi_{2})^{\downarrow Y}(\langle\mathbf{x}{}^{\downarrow X\cap Y},\mathbf{z}\rangle)=\\
 & = & \phi_{1}^{\downarrow X\cap Y}(\mathbf{x}^{\downarrow X\cap Y})\cdot\phi_{2}(\langle\mathbf{x}^{\downarrow X\cap Y},\mathbf{z}\rangle)
\end{eqnarray*}
 Hence 
\begin{eqnarray*}
M\cdot M & = & \phi_{1}(\mathbf{x})\cdot\phi_{2}^{\downarrow X\cap Y}(\mathbf{x}^{\downarrow X\cap Y})\cdot\phi_{1}^{\downarrow X\cap Y}(\mathbf{x}^{\downarrow X\cap Y})\cdot\phi_{2}(\langle\mathbf{x}^{\downarrow X\cap Y},\mathbf{z\rangle})\\
 & = & \phi_{1}(\mathbf{x})\cdot\phi_{2}(\langle\mathbf{x}^{\downarrow X\cap Y},\mathbf{z}\rangle)\cdot\phi_{2}^{\downarrow X\cap Y}(\mathbf{x}^{\downarrow X\cap Y})\cdot\phi_{1}^{\downarrow X\cap Y}(\mathbf{x}^{\downarrow X\cap Y})\\
 & = & \phi(\langle\mathbf{x},\mathbf{z}\rangle)\ \ \ \ \ \ \ \ \ \ \ \ \ \ \ \ \ \ \cdot\phi(\conf{x}^{\downarrow X\cap Y})\\
 & = & \phi(\langle\mathbf{x},\mathbf{z}\rangle)\cdot M
\end{eqnarray*}

Now, we have that $M\neq0$ and that both $M$ and $\phi(\langle\mathbf{x},\mathbf{z}\rangle)$
are in $Im(\cdot),$ since by definition $\phi=\phi_{1}\times\phi_{2}.$
By applying that $(R,+,\cdot)$ is square multiplicatively cancellative
on image we have that $\phi(\langle\mathbf{x},\mathbf{z}\rangle)=M,$
which proves $\langle\mathbf{x},\mathbf{z}\rangle\in c_{\phi}.$\end{proof}
\begin{thm}
Let $(\Phi,U)$ be a valuation algebra induced by a selective commutative
semiring $(R,+,\cdot)$. If the valuation algebra has two variables
that can take two or more values, and projective extensibility holds
on $\Phi$, then $(R,+,\cdot)$ is square ordered. \end{thm}
\begin{proof}
To prove it we will generalize counterexample \ref{cnt:NotSquareOrdered}.
Assume the semiring is not square ordered. This means that there are
$a,b\in R$ such that $a\cdot a=b\cdot a$ and $b\cdot b<a\cdot a.$
Let $x,y$ be two variables with two or more variables. \\
Define $\phi_{X}(x)=\begin{cases}
b & \mbox{ if }x=0\\
a & \mbox{ if }x=1\\
0 & \mbox{ otherwise.}
\end{cases}$ and $\phi_{Y}(y)=\begin{cases}
b & \mbox{ if }y=0\\
a & \mbox{ if }y=1\\
0 & \mbox{ otherwise.}
\end{cases}$ \\
Let $\phi=\phi_{X}\times\phi_{Y}.$ We have that $\phi(x,y)=\begin{cases}
b\cdot b & \mbox{ if }x=y=0\\
a\cdot a & \mbox{ if }(x,y)\in\{(1,1),(0,1),(1,0)\}\\
0 & \mbox{ otherwise.}
\end{cases}$

Now, from the definition of projection, $\phi^{\downarrow X}(x)=\begin{cases}
a\cdot a & \mbox{ if }x=0\mbox{ or }x=1\\
0 & \mbox{ otherwise.}
\end{cases}$ and $\phi^{\downarrow Y}(y)=\begin{cases}
a\cdot a & \mbox{ if }y=0\mbox{ or }y=1\\
0 & \mbox{ otherwise.}
\end{cases}$ 

Clearly projective completability does not hold in this example, since
the solutions are $\{\{x\mapsto0,y\mapsto1\},\{x\mapsto1,y\mapsto0\},\{x\mapsto1,y\mapsto1\}\}$
and by projective completability we also find $\{x\mapsto0,y\mapsto0\}$.
\end{proof}

\subsection{Piecewise completability on optimization problems}

In Theorem \ref{thm:Extend-to-Subtree-Sufficient} we have shown that
piecewise completability is the sufficient condition for \ESName\ 
solving the partial optimization SFP. Furthermore, non-empty piecewise
completability is the necessary condition for \SESName\  solving
the single SFP. In this section we show that the optimization extension
system guarantees non-empty piecewise completability. As a consequence,
we can use \SESName\  to find a solution and \ESName\  to find some
solutions in any optimization SFP. 
\begin{thm}
\label{thm:optimizationSSFP}The optimization extension system satisfies
non-empty piecewise completability on $\Phi$.\end{thm}
\begin{proof}
Take a valuation $\phi=\xi_{1}\times\xi_{2}$ where $\xi_{1},\xi_{2}\in\Phi$,
with domains $X$ and $Y$ respectively. We have to prove that $CO(c_{\phi^{\downarrow X}},\xi_{2})\subseteq c_{\phi}.$ 

By definition of set of completions, we have that 
\[
CO(c_{\phi^{\downarrow X}},\xi_{2})=\{\langle\mathbf{x},\mathbf{z}\rangle|\mathbf{x}\in c_{\phi^{\downarrow X}}\text{ and }\mathbf{z}\in W_{\xi_{2}}^{X\cap Y}(\mathbf{x}^{\downarrow X\cap Y})\}.
\]

Now, from the definition of $c_{\phi},$ we have that 
\[
c_{\phi}=W_{\phi}^{\emptyset}(\diamond)=\{\langle\mathbf{x},\mathbf{z\rangle})|\mathbf{x}\in c_{\phi^{\downarrow X}}\text{ and }\mathbf{z}\in W_{\xi_{1}\times\xi_{2}}^{X}(\mathbf{x})\}.
\]

So, piecewise completability is satisfied if, and only if, for each
$\mathbf{x}\in c_{\phi^{\downarrow X}},$ we have that $W_{\xi_{2}}^{X\cap Y}(\mathbf{x}^{\downarrow X\cap Y})\subseteq W_{\xi_{1}\times\xi_{2}}^{X}(\mathbf{x}).$

Now take $\mathbf{z}\in W_{\xi_{2}}^{X\cap Y}(\mathbf{x}^{\downarrow X\cap Y}).$
From the definition we have that $\xi_{2}(\langle\mathbf{x}^{\downarrow X\cap Y},\mathbf{z}\rangle)=\xi_{2}^{\downarrow X\cap Y}(\mathbf{x}^{\downarrow X\cap Y}).$
Multiplying by $\xi_{1}(\mathbf{x})$ we get that 
\begin{equation}
\xi_{1}(\mathbf{x})\cdot\xi_{2}(\langle\mathbf{x}^{\downarrow X\cap Y},\mathbf{z}\rangle)=\xi_{1}(\mathbf{x})\cdot\xi_{2}^{\downarrow X\cap Y}(\mathbf{x}^{\downarrow X\cap Y}).\label{eq:Xi-2-Marginal}
\end{equation}

Next, we have that
\begin{eqnarray*}
(\xi_{1}\times\xi_{2})(\mathbf{x},\mathbf{z}) & = & \xi_{1}(\mathbf{x})\cdot\xi_{2}(\langle\mathbf{x}^{\downarrow X\cap Y},\mathbf{z}\rangle)=\xi_{1}(\mathbf{x})\cdot\xi_{2}^{\downarrow X\cap Y}(\mathbf{x}^{\downarrow X\cap Y})\\
 & = & (\xi_{1}\times\xi_{2}^{\downarrow X\cap Y})(\mathbf{x})=(\xi_{1}\times\xi_{2})^{\downarrow X}(\mathbf{x}).
\end{eqnarray*}

where the first equality is the definition of combination, the second
is equation \ref{eq:Xi-2-Marginal}, and the remaining are basic valuation
algebra manipulations

But now, since $W_{\xi_{1}\times\xi_{2}}^{X}(\mathbf{x})=$$\{\mathbf{z}\in\Omega_{d(\phi)-X}|(\xi_{1}\times\xi_{2})(\mathbf{x},\mathbf{z})=(\xi_{1}\times\xi_{2})^{\downarrow X}(\mathbf{x})\}$,
it is clear that $\mathbf{z}\in W_{\xi_{1}\times\xi_{2}}^{X}(\mathbf{x}).$
Hence , $W_{\xi_{2}}^{X\cap Y}(\mathbf{x}^{\downarrow X\cap Y})\subseteq W_{\xi_{1}\times\xi_{2}}^{X}(\mathbf{x}).$

Since the optimization extension system is a total implementation
of a solution concept which is guaranteed non-empty, $c_{\phi^{\downarrow X}}$
is non-empty and by the definition of extension set in equation \ref{eq:optimization-Extension},
the set of extensions cannot be empty, hence $CO(c_{\phi^{\downarrow X}},\xi_{2})$
is non-empty. This provides guaranteed non-empty piecewise completability. 
\end{proof}

As a result of Theorem \ref{thm:optimizationSSFP}, non-empty piecewise
completability is guaranteed on any optimization extension system.
In Theorem \ref{thm:Extend-to-Subtree-Necessary} we have seen that
total piecewise completability is a necessary and sufficient condition
for \ESName\  to solve the complete SFP. We are interested in characterizing
for which semirings does the optimization extension system satisfy
total piecewise completability. Theorem \ref{thm:optimizationSSFP}
proved non-empty piecewise completability on $\text{\ensuremath{\Phi}}.$
It turns out that it is not possible to extend this result to total
piecewise completability. However, sometimes the completability conditions
only hold for a subset of the valuations and this is the case here.
We will see that total piecewise completability does only hold if
the valuation $\phi$ for which we try to find a solution has at least
one configuration whose value is not zero. However to do that first
we take a detour to talk about valuation algebras with null elements.
\begin{defn}
\label{def:Null-Element}An element $0_{X}\in\Phi_{X}$ is a \emph{null
element }if 
\begin{enumerate}
\item For each $\phi\in\Phi_{X},$ we have $\phi\times0_{X}=0_{X}\times\phi=0_{X}.$ 
\item For $X\subseteq Y\subseteq U$ and $\phi\in\Phi_{Y},$ we have that
$\phi^{\downarrow X}=0_{X}$ if and only if $\phi=0_{Y}.$
\end{enumerate}
\end{defn}

\begin{lem}
In a valuation algebra, the set $NN\subseteq\Phi$ of non-null elements
is projection-closed and combination-breakable. \end{lem}
\begin{proof}
From the second condition in Definition \ref{def:Null-Element}, we
have that the set of non-null elements is projection-closed. To prove
that it is combination breakable pick any $\phi$ that is non-null
and such that $\phi=\xi_{1}\times\xi_{2}.$ Now assume that either
$\xi_{1}$ or $\xi_{2}$ is null. Then by the first condition in Definition
\ref{def:Null-Element} we have that $\phi$ is null which is a contradiction.
Hence, both $\xi_{1}$ and $\xi_{2}$ must be non-null.
\end{proof}
In selective semiring induced valuation algebras, a valuation $\phi$
is null if and only if $\phi^{\downarrow\emptyset}(\diamond)=0.$
Thus, all possible configurations in $\mbox{\ensuremath{\Omega}}_{d(\phi)}$
are solutions. Since \noun{\ESName\ } runs the \noun{Collect} algorithm
as a previous step, it is easy to determine whether $\phi$ is constant
$0$ by assessing $\phi^{\downarrow\emptyset}(\diamond)=\left(\psi_{r}'\right)^{\downarrow\emptyset}(\diamond)$
and checking whether it is equal to $0.$ In that case we can directly
return $\Omega_{d(\phi)}$. Thus, we can easily identify and solve
null valuations. So, we have to concentrate on when does total piecewise
completability hold on $NN.$ Next, we define weak multiplicative
cancellativity and prove that it is the sufficient and necessary condition
on a semiring for \ESName\ to solve the complete optimization SFP.
\begin{defn}
A commutative semiring $R$ is weakly multiplicatively cancellative
if for any $a,b,c\in R,$ we have that 
\[
a\cdot c\neq0\mbox{, and }a\cdot c=b\cdot c\,\,\,\,\,\,\,\,\,\,\,\,\,\,\,\,\,\,\,\mbox{ implies that \,\,\,\,\,\,\,\,\,\,\,\,\,\,\,\,\,\,\,\,\ensuremath{a=b}.}
\]
\end{defn}
\begin{thm}
Let $R$ be a commutative selective semiring. If $R$ is weakly multiplicatively
cancellative then its induced valuation algebra satisfies total piecewise
completability on $NN$. On the other hand, if the valuation algebra
has one variable that can take two or more values, and total piecewise
completability on $NN$ is satisfied, then $R$ is weakly multiplicatively
cancellative. \end{thm}
\begin{proof}
We start proving that, if the semiring is weakly multiplicatively
cancellative, we have total piecewise completability on $NN.$ Take
a valuation $\phi=\xi_{1}\times\xi_{2}$ where $\xi_{1},\xi_{2}\in\Phi$,
with domains $X$ and $Y$ respectively. We have to prove that $CO(c_{\phi^{\downarrow X}},\xi_{2})=c_{\phi}.$ 

By definition of set of completions, we have that 
\[
CO(c_{\phi^{\downarrow X}},\xi_{2})=\{(\mathbf{x},\mathbf{z})|\mathbf{x}\in c_{\phi^{\downarrow X}}\text{ and }\mathbf{z}\in W_{\xi_{2}}^{X\cap Y}(\mathbf{x}^{\downarrow X\cap Y})\}.
\]

Now, from the definition of $c_{\phi},$ we have that 
\[
c_{\phi}=W_{\phi}^{\emptyset}(\diamond)=\{(\mathbf{x},\mathbf{z})|\mathbf{x}\in c_{\phi^{\downarrow X}}\text{ and }\mathbf{z}\in W_{\xi_{1}\times\xi_{2}}^{X}(\mathbf{x})\}.
\]

So, total piecewise completability is satisfied if, and only if, for
each $\mathbf{x}\in c_{\phi^{\downarrow X}},$ we have that $W_{\xi_{2}}^{X\cap Y}(\mathbf{x}^{\downarrow X\cap Y})=W_{\xi_{1}\times\xi_{2}}^{X}(\mathbf{x}).$ 

In Theorem \ref{thm:optimizationSSFP}, we proved that $W_{\xi_{2}}^{X\cap Y}(\mathbf{x}^{\downarrow X\cap Y})\subseteq W_{\xi_{1}\times\xi_{2}}^{X}(\mathbf{x}).$
Thus, it remains to prove that $W_{\xi_{1}\times\xi_{2}}^{X}(\mathbf{x})\subseteq W_{\xi_{2}}^{X\cap Y}(\mathbf{x}^{\downarrow X\cap Y}).$

Now take $\mathbf{z}\in W_{\xi_{1}\times\xi_{2}}^{X}(\mathbf{x}).$
From the definition we have that $(\xi_{1}\times\xi_{2})(\mathbf{x},\mathbf{z})=(\xi_{1}\times\xi_{2})^{\downarrow X}(\mathbf{x}).$
From here,

\[
(\xi_{1}\times\xi_{2})(\mathbf{x},\mathbf{z})=(\xi_{1}\times\xi_{2}^{\downarrow X\cap Y})(\mathbf{x})
\]

and by definition of combination 
\[
\xi_{1}(\mathbf{x})\cdot\xi_{2}((\mathbf{x}^{\downarrow X\cap Y},\mathbf{z}))=\xi_{1}(\mathbf{x})\cdot\xi_{2}^{\downarrow X\cap Y}(\mathbf{x}).
\]

We have that $\xi_{1}(\mathbf{x})\cdot\xi_{2}^{\downarrow X\cap Y}(\mathbf{x})=\phi^{\downarrow X}(\mathbf{x})=\phi^{\downarrow\emptyset}(\diamond)\neq0,$
where the second equality follows because $\mathbf{x}\in c_{\phi^{\downarrow X}},$
and the third one since $\phi\in NN,$ and hence $\phi^{\downarrow\emptyset}(\diamond)\neq0.$
So we can apply weak cancellation to $\xi_{1}(\mathbf{x})$ getting
\[
\xi_{2}((\mathbf{x}^{\downarrow X\cap Y},\mathbf{z}))=\xi_{2}^{\downarrow X\cap Y}(\mathbf{x}^{\downarrow X\cap Y}).
\]

But this is exactly the condition that $\mathbf{z}$ has to satisfy
in order to be in $W_{\xi_{2}}^{X\cap Y}(\mathbf{x}^{\downarrow X\cap Y}).$
We have proven that $W_{\xi_{1}\times\xi_{2}}^{X}(\mathbf{x})\subseteq W_{\xi_{2}}^{X\cap Y}(\mathbf{x}^{\downarrow X\cap Y})$
and in Theorem \ref{thm:optimizationSSFP}, we proved that $W_{\xi_{2}}^{X\cap Y}(\mathbf{x}^{\downarrow X\cap Y})\subseteq W_{\xi_{1}\times\xi_{2}}^{X}(\mathbf{x}).$
Thus, $W_{\xi_{2}}^{X\cap Y}(\mathbf{x}^{\downarrow X\cap Y})=W_{\xi_{1}\times\xi_{2}}^{X}(\mathbf{x}).$ 

The second part of the proof assumes total piecewise completability
on $NN$ and concludes that the semiring must be weakly multiplicatively
cancellative. To prove it, we assume that it is not and will reach
a contradiction. Let $a,b,c\in R,$ such that $a\cdot c\neq0,$ $a\cdot c=b\cdot c$
and $a\neq b.$ Now, we build a valuation which has as domain a single
variable $\mathbf{x}$ with at least two values, namely $x_{0}$ and
$x_{1}.$ 
\[
\phi(\mathbf{x})=(\xi_{1}\times\xi_{2})(\mathbf{x})
\]
with
\[
\xi_{1}(\diamond)=c,
\]
and 
\[
\xi_{2}(\mathbf{x})=\begin{cases}
a & \mbox{\ if \ensuremath{\mathbf{x}=x_{0}},}\\
b & \ \mbox{if \ensuremath{\mathbf{x}=x_{1},}}\\
0 & \ \mbox{otherwise.}
\end{cases}
\]

Note that $\phi$ is non-null and that the set of solutions of $\phi$
is $\{x_{0},x_{1}\}.$ Now, if $a>b,$ then $W_{\xi_{2}}^{\emptyset}(\diamond)=\{x_{0}\},$
and the only solution found by piecewise completing will be $\{x_{0}\}.$
On the other hand, if $b>a$, then then $W_{\xi_{2}}^{\emptyset}(\diamond)=\{x_{1}\},$
and the only solution found by piecewise completing will be $\{x_{1}\}.$
Thus, in both cases we get to a contradiction.
\end{proof}
Table \ref{tab:ImpactOpt-1} summarizes the results in this section,
providing the sufficient and necessary conditions for each algorithm.
We have proven a sufficient condition to \EGPName,  which correctly
deals with counterexample \ref{cnt:NotSquareOrdered}. For \ESName\ 
to solve the complete optimization SFP, Pouly and Kohlas required
strict monotonicity which, for selective semirings, is equivalent
to multiplicative cancellativity (see Proposition \ref{prop:Strict-Iff-Cancellative}).
We have proven that weakly multiplicative cancellativity suffices.
Furthermore, where possible, we have provided also necessary conditions.

\begin{table}
\begin{centering}
{\small{}}%
\begin{tabular}{|>{\centering}p{1.7cm}|>{\centering}m{2.3cm}|>{\centering}m{3.2cm}|>{\centering}m{3.2cm}|}
\hline 
\textbf{\scriptsize{}Algorithm } & \textbf{\scriptsize{}Problem} & \textbf{\scriptsize{}Semiring suff. cond.} & \textbf{\scriptsize{}Semiring nec. cond.}\tabularnewline
\hline 
\hline 
\noun{\scriptsize{}\ref{alg:ExtendAll}} & {\scriptsize{}Complete optimization SFP} & {\scriptsize{}Square multiplicatively cancellative on image} & {\scriptsize{}Square ordered}\tabularnewline
\hline 
\noun{\scriptsize{}\ref{alg:ExtendOne}} & {\scriptsize{}Single optimization SFP} & {\scriptsize{}None} & {\scriptsize{}None}\tabularnewline
\hline 
\noun{\scriptsize{}\ref{alg:ExtendSome}} & {\scriptsize{}Partial optimization SFP} & {\scriptsize{}None} & {\scriptsize{}None}\tabularnewline
\hline 
\noun{\scriptsize{}\ref{alg:ExtendSome}} & {\scriptsize{}Complete optimization SFP} & {\scriptsize{}Weakly multiplicatively cancellative } & {\scriptsize{}Weakly multiplicatively cancellative }\tabularnewline
\hline 
\end{tabular}
\par\end{centering}{\small \par}

\protect\caption{\label{tab:ImpactOpt-1}Sufficient and necessary conditions on optimization
problems (semirings considered are always commutative and selective).}
\end{table}

\section{Conclusions\label{sec:Conclusions}}

The theory for the generic construction of solutions in valuation
based systems \cite{Pouly2011c,Pouly2011a} studies three widely used
dynamic programming algorithms from the most general perspective and
provides necessary conditions for those algorithms to be correct.
We have presented counterexamples to the results presented there and
we have shown that the counterexamples have a deep impact in the theory.
This has opened the way for identifying two properties of extension
systems: projective completability and piecewise completability. We
have proven that such properties constitute sufficient and necessary
conditions for those generic algorithms to be correct, allowing for
a sharper characterization of when each algorithmic scheme can be
applied. To the best of our knowledge, up to know no necessary conditions
for these generic algorithms had been presented in the literature.

A particularly interesting case where these algorithms can be applied
is valuation algebras induced by a commutative selective semiring,
where they constitute the base of well known optimization algorithms.
For that case, we have also corrected a result in \cite{Pouly2011c,Pouly2011a}.
Furthermore, we have been able to translate the sufficient and necessary
conditions for the algorithms into conditions for the semiring, identifying
three new semiring properties: square multiplicatively cancellative
on image, square ordered and weakly multiplicatively cancellative.
Although we have started scratching the relationships between these
semiring properties, a deeper study of their interactions remains
as future work. 

As a result, our corrected theory provides the more general description
of these generic algorithms and the sharpest characterization to date
of their necessary and sufficient conditions.

\section*{Acknowledgements}

The authors would like to thank Professor Jürg Kohlas for his many
valuable comments, suggestions and discussions along the craft of
this paper. This work has been supported by projects COR (TIN2012-38876-C02-01),
GEAR (CSIC - 201350E112) and by the Generalitat of Catalunya grant
2009-SGR-1434.

\appendix
\renewcommand*{\appendixname}{}
\section{Properties of rooted covering join trees}

We prove some poperties of rooted covering join trees which are needed
to ease the proofs of the results presented in the paper. In any rooted
join tree, for each node $i\in V,$ we define $\lambda^{de}(i)$ as
the set of variables that appear in the scope of the descendants of
$i$, namely $\lambda^{de}(i)=\bigcup_{j\in de(i)}\lambda(j)$. Furthermore
we define $\lambda^{nde}(i)$ as the set of variables that appear
in the scope of the non-descendants of $i,$ namely $\lambda^{nde}(i)=\bigcup_{j\in nde(i)}\lambda(j)$. 
\begin{lem}
For any node $i$ of $V$

\textup{
\begin{equation}
s_{i}=\lambda(i)\cap\lambda^{nde}(i)\label{eq:separatorIsIntersectionWithND}
\end{equation}
}
\begin{equation}
\bigcup_{j\in ch(i)}s_{j}=\lambda(i)\cap\lambda^{de}(i)\label{eq:intersectionWithSubtree}
\end{equation}

\textup{
\begin{equation}
\lambda^{nde}(i)\cap\lambda^{de}(i)\subseteq\lambda(i)\cap\lambda^{de}(i)\label{eq:lambdaLargerThanIntersection}
\end{equation}
}\end{lem}
\begin{proof}
We start proving equation \ref{eq:separatorIsIntersectionWithND}.
If $i$ is the root, then $s_{i}=\emptyset$ and the equation is trivially
satisfied. Assume that $i$ is not the root. By definition of $\lambda^{nde}(i)$,
we have that $\lambda(i)\cap\lambda^{nde}(i)=\lambda(i)\cap\bigcup_{j\in nde(i)}\lambda(j)=\bigcup_{j\in nde(i)}(\lambda(i)\cap\lambda(j))=s_{i}\cup\bigcup_{j\in nde(i)-\{p_{i}\}}(\lambda(i)\cap\lambda(j)).$
For any $j\in nde(i)-\{p_{i}\},$ we have that $p_{i}$ lies in the
path between $j$ and $i$, and by the running intersection property,
$\lambda(i)\cap\lambda(j)\subseteq\lambda(p_{i})$. Since $\lambda(i)\cap\lambda(j)\subseteq\lambda(i),$
we have that $\lambda(i)\cap\lambda(j)\subseteq\lambda(i)\cap\lambda(p_{i})=s_{i}.$
Thus, $\bigcup_{j\in nde(i)-\{p_{i}\}}\lambda(i)\cap\lambda(j)\subseteq s_{i},$
and $\lambda(i)\cap\lambda^{nde}(i)=s_{i}.$

Next, we will prove equation \ref{eq:intersectionWithSubtree}
\begin{eqnarray*}
\lambda(i)\cap\lambda^{de}(i) & = & \lambda(i)\cap\bigcup_{j\in de(i)}\lambda(j)=\lambda(i)\cap\bigcup_{j\in ch(i)}[\lambda(j)\cup\bigcup_{k\in de(j)}\lambda(k)]\\
 & = & \bigcup_{j\in ch(i)}[(\lambda(i)\cap\lambda(j))\cup\bigcup_{k\in de(j)}(\lambda(i)\cap\lambda(k))].
\end{eqnarray*}
Now, by the running intersection property, $\lambda(i)\cap\lambda(k)\subseteq\lambda(i)\cap\lambda(j),$
so we can remove the union $\bigcup_{k\in de(j)}(\lambda(i)\cap\lambda(k))$
leaving

\[
\lambda(i)\cap\lambda^{de}(i)=\bigcup_{j\in ch(i)}[\lambda(i)\cap\lambda(j)]=\bigcup_{j\in ch(i)}s_{j}.
\]

Finally, we will conclude by proving equation \ref{eq:lambdaLargerThanIntersection}.
Applying the definitions we have that $\lambda^{nde}(i)\cap\lambda^{de}(i)=\left(\bigcup_{j\in nde(i)}\lambda(j)\right)\cap\left(\bigcup_{k\in de(i)}\lambda(k)\right)=\bigcup_{j\in nde(i)}\bigcup_{k\in de(i)}\lambda(j)\cap\lambda(k).$
But now node $i$ lies in the path between any node $j$ which is
non-descedant of $i$ and any other node $k$ which is descendant
of $i.$ Thus, by the running intersection property we have that $\lambda(j)\cap\lambda(k)\subseteq\lambda(i),$
and that $\bigcup_{j\in nde(i)}\bigcup_{k\in de(i)}\lambda(j)\cap\lambda(k)\subseteq\lambda(i).$
From here we have that $\lambda^{nde}(i)\cap\lambda^{de}(i)\subseteq\lambda(i)\cap\lambda^{de}(i).$
\end{proof}

\section{Minimally labeled covering join trees\label{sec:MinimallyLabeledJoinTrees}}

In the paper we make the assumption that covering join trees are minimally
labeled (see Assumption \ref{assu:Minimal}). In this appendix we
start by checking that, for a fixed tree, there is no covering join
tree whose labels are smaller that those of a minimally labeled join
tree. Afterwards, we prove that it is easy to build a minimally labeled
covering join tree provided a tree and a valuation assignment function.
Finally we prove some properties of minimally labeled covering join
trees which are used in the proofs in the paper.

We start by proving that there can be no labelling smaller than that
of a minimally labeled covering join tree. 
\begin{lem}
\label{lem:LambdaBiggerThanDomain}Let ($V,E)$ be a tree. Given a
valuation $\phi=\phi_{1}\times\cdots\times\phi_{n},$ there is no
covering join tree $\mathcal{T}=(V,E,\lambda,U)$, valuation assignment
$a$, $i\in V,$ and $k\in ne(i),$ such that $\lambda(i)\nsupseteq d(\psi_{i})\cup\bigcup_{j\in ne(i)\setminus\{k\}}s_{ij}.$\end{lem}
\begin{proof}
The proof is immediate since $s_{ij}=\lambda(i)\cap\lambda(j)\subseteq\lambda(i)$
and $d(\psi_{i})\subseteq\lambda(i)$ is required for $a$ to be a
valuation assignment. 
\end{proof}
Now, Algorithm \ref{alg:MinimalLambdas} provides a procedure to assess
a minimally labeled covering join tree provided a covering join tree
and a valuation assignment $a.$

\begin{algorithm}
{\footnotesize{}\begin{algorithmic}[1]
\ForAll{nodes $i$  of $\mathcal T$}
	\State $\alpha(i) := \bigcup_{j\in a^{-1}(i)}d(\phi_j)$
	\State $\beta(i) := \alpha(i)$
	\State $\gamma(i) := \emptyset$
\EndFor
\ForAll{nodes $i$ of $\mathcal T$ except the root in an upward order }
	\State $\alpha(p_i) := \alpha(p_i) \cup \alpha(i)$
    \State $\beta(p_i) := \beta(p_i) \cup (\gamma(p_i) \cap \alpha(i))$
	\State $\gamma(p_i) := \gamma(i) \cup \alpha(i)$
\EndFor
\State $\lambda(r) := \beta(r)$
\ForAll{nodes $i$ of $\mathcal T$ except the root in an downward order }
	\State $\lambda(i) := \beta(i) \cup (\lambda(p_i) \cap \alpha(i))$
\EndFor
\State \Return $\boldsymbol{\lambda}$;
\end{algorithmic}}{\footnotesize \par}

\protect\caption{\label{alg:MinimalLambdas}\noun{MinimalLambdas} algorithm}
\end{algorithm}

\begin{lem}
\noun{MinimalLambdas }asseses a minimally labeled covering join tree.
Furthermore \noun{MinimalLambdas }only requires time and space $O(|V||U|),$
where $V$ is the set of nodes of the join tree and $U$ is the set
of variables of the problem.\end{lem}
\begin{proof}
We will start proving that \noun{MinimalLambdas} asseses a covering
join tree. Let $\phi=\phi_{1},\dots,\phi_{n}$ be a valuation and
let $\mathcal{\mathcal{T}}=(V,E,\lambda,U)$ be a tree where the \noun{MinimalLambdas}
algorithm has been run. After the second loop we have

{\footnotesize{}
\begin{eqnarray}
\begin{array}{c}
\alpha(i)=d(\psi_{i})\cup\left(\bigcup_{j\in de(i)}d(\psi_{j})\right)\end{array} & \mbox{ and } & \beta(i)=d(\psi_{i})\cup\left(\bigcup_{\substack{{j,k\in ch(i)}\\
j\neq k
}
}\alpha(j)\cap\alpha(k)\right).\label{eq:alpha-beta}
\end{eqnarray}
}Notice that $\beta(i)\subseteq\alpha(i)$ and as a consequence $\lambda(i)=\beta(i)\cup\left(\lambda(p_{i})\cap\alpha(i)\right)\subseteq\alpha(i)\cup\left(\lambda(p_{i})\cap\alpha(i)\right)\subseteq\alpha(i).$
Also notice that for any $i\in de(j)$, $\alpha(i)\subseteq\alpha(j).$
After the third loop we get $\lambda(i)=\beta(i)\cup(\lambda(p_{i})\cap\alpha(i)).$
Thus, $\lambda(p_{i})\cap\alpha(i)\subseteq\lambda(i).$

Next we will prove that the runing intersection property is satisfied.
Let $n_{1},n_{2},\dots,n_{m}$ be the unique path between two given
nodes $n_{1},n_{m}\in V$. We want to see that $\lambda(n_{1})\cap\lambda(n_{m})\subseteq\lambda(n_{i}),$
for $1<i<m$. Notice that if $m\leq2$ it is trivially true. Therefore
we will suppose $m\geq3$. As long as $\mathcal{T}$ is a tree, the
previous path can be seen as the composition of two different paths,
one ascending path which grows up from $n_{1}$ up to $n_{i}$ with
$1\leq i\leq m$, and one descending path from $n_{i}$ to $n_{m}$.
That is $n_{j+1}=p_{n_{j}}$ for $1\leq j\leq i-1$ and $n_{j+1}\in ch(n_{j})$
for $i\leq j\leq m-1.$ Notice that if $n_{1}\neq n_{m}$ at most
one of these subpaths may be empty. Hence there are three possible
configurations for the paths either the descending path is empty,
or the ascending path is empty or no subpath is empty. Equivalently,
either $i=m$ or $i=1$ or $1\lneq i\lneq m$.
\begin{enumerate}
\item Assume that the descending path is empty, so $i=m$. We have that
$\lambda(n_{1})\cap\lambda(n_{m})\subseteq\lambda(n_{1})\subseteq\alpha(n_{1})\subseteq\alpha(n_{m-1})$
and $\lambda(n_{1})\cap\lambda(n_{m})\subseteq\lambda(n_{m})=\lambda(p_{n_{m-1}}).$
In conclusion we obtain $\lambda(n_{1})\cap\lambda(n_{m})\subseteq\alpha(n_{m-1})\cap\lambda(p_{n_{m-1}})\subseteq\lambda(n_{m-1})$,
and we have verificed that $n_{m-1}$ fulfills the condition. Now
since, by induction we have that $\lambda(n_{1})\cap\lambda(n_{m})\subseteq\lambda(n_{j}),$
for $1<j<m.$ 
\item In case the ascending path is empty we can consider the path from
$n_{m}$ to $n_{1}$ and use the previous argument, since $n_{m},n_{m-1},\dots,n_{1}$
is an ascending path.
\item Finally, if no subpath is empty it holds that $1\lneq i\lneq m.$
In this case, we have that $\lambda(n_{1})\subseteq\alpha(n_{1})\subseteq\alpha(n_{i-1})$
and $\lambda(n_{m})\subseteq\alpha(n_{m})\subseteq\alpha(n_{i+1})$.
Since $n_{i-1},n_{i+1}\in ch(n_{i})$ we obtain: $\lambda(n_{1})\cap\lambda(n_{m})\subseteq\alpha(n_{i-1})\cap\alpha(n_{i+1})\subseteq{\displaystyle \bigcup_{\substack{{j,k\in ch(i)}\\
j\neq k
}
}\alpha(j)\cap\alpha(k)\subseteq\beta(i)\subseteq\lambda(i)},$ and thus the condition is fulfilled for $i.$ As long as $n_{1},\dots,n_{i}$
is an ascending path and $n_{i},\dots,n_{m}$ is a descending path,
we can use the previous cases to check that $\lambda(n_{1})\cap\lambda(n_{i})\subseteq\lambda(n_{j})$
for $1\leq j\leq m,$ which concludes the proof since $\lambda(n_{1})\cap\lambda(n_{m})=\lambda(n_{1})\cap\lambda(n_{m})\cap\lambda(n_{i})\subseteq\lambda(n_{1})\cap\lambda(n_{i})\subseteq\lambda(n_{j}).$
\end{enumerate}
We have just shown that $\mathcal{T}$ is a join tree, but as long
as for all $\psi(i)$ it is satisfied $d(\psi_{i})\subseteq\lambda(i)$
we also have that $\mathcal{T}$ is actually a covering join tree.

Next, we will prove that $\mathcal{T}$ is minimally labeled. By lemma
\ref{lem:LambdaBiggerThanDomain} we already know that $\lambda(i)\supseteq d(\psi_{i})\cup\bigcup_{j\in ne(i)\setminus\{k\}}s_{ij}$
for all $i\in V$. We will prove by contradiction that $\lambda(i)=d(\psi_{i})\cup\bigcup_{j\in ne(i)\setminus\{k\}}s_{ij}.$
Assume that for some $i\in V$ there is a variable $x\in\lambda(i)$
and a $k'\in ne(i)$ such that $x\notin d(\psi_{i})\cup\bigcup_{j\in ne(i)\setminus\{k'\}}s_{ij}.$
Since $\lambda(i)=\beta(i)\cup(\lambda(p_{i})\cap\alpha(i))$, we
have that $x\in\beta(i)$ or $x\in\lambda(p_{i})\cap\alpha(i).$ 
\begin{enumerate}
\item If $x\in\beta(i),$ since by assumption $x\notin d(\psi_{i}),$ we
can conclude from equation \ref{eq:alpha-beta} that $x\in\bigcup_{\substack{{j,k\in ch(i)}\\
j\neq k
}
}\alpha(j)\cap\alpha(k)$. Nonetheless, if there are $j,k\in ch(i)$ such that $x\in\lambda(j)$
and $x\in\lambda(k)$, then by the running intersection property $x\in\lambda(i)$
and as a consequence $x\in s_{i,j}$ and $x\in s_{ik}.$ For any possible
value of $k',$ either $s_{ij}$ or $s_{ik}$ will be part of $\bigcup_{j\in ne(i)\setminus\{k'\}}s_{ij}.$,
and thus, we have a contradiction. 
\item If $x\in\lambda(p_{i})\cap\alpha(i)$. We have $x\in\alpha(i)=d(\psi_{i})\cup\left(\bigcup_{j\in de(i)}d(\psi_{j})\right)$
and $x\in\lambda(p_{i})$. Since by assumption $x\notin d(\psi_{i}),$
then $x\in\bigcup_{j\in de(i)}d(\psi_{j}).$ Nevertheless, $x\in\bigcup_{j\in de(i)}d(\psi_{j})$
implies $x\in\bigcup_{j\in de(i)}\lambda(j)$, and by the running
intersection property, there must exist at least one $j\in ch(i)$
such that $x\in\lambda(j)$, in particular $x\in s_{i,j}$, which
also contradicts our hypothesis since $x\in\lambda(p_{i})$ implies
$x\in s_{i,p_{i}}$.
\end{enumerate}
\end{proof}

\subsection{Basic properties of minimally labeled covering join trees}

In the following, let $\mathcal{\mathcal{T}}=(V,E,\lambda,U)$ be
a\emph{ }minimally labeled covering join tree\emph{.} 

Removing any edge $\{i,j$\} on the $\mathcal{T}$, breaks it into
two different trees: $\mathcal{T}_{i}^{-j},$ the one containing $i$
and $\mathcal{T}_{j}^{-i},$the one containing $j.$
\begin{lem}
\label{lem:Tree-Breakup}For any edge $\{i,j\}$ of $\mathcal{T}$, 

\textup{
\begin{equation}
\bigcup_{k\in\mathcal{T}_{i}^{-j}}\lambda(k)=d(\prod_{k\in\mathcal{T}_{i}^{-j}}\psi_{k})\label{eq:lambdaDIsProduct-1-1}
\end{equation}
}\end{lem}
\begin{proof}
We can place $i$ at the root and use induction on the height of the
tree. 

If $i$ is a leaf, then it is trivially true, since both sides are
$d(\psi_{i})$. 

Let $i$ be a node with height $n$ and assume it is true whenever
the height is smaller than $n.$ Each node in $\mathcal{T}_{i}$ lies
in a subtree rooted at one of the children of $i,$ so $\bigcup_{k\in\mathcal{T}_{i}^{-j}}\lambda(k)=\lambda(i)\cup\left(\bigcup_{j'\in ne(i)-\{j\}}\bigcup_{k'\in\mathcal{T}_{j'}^{-i}}\lambda(k')\right)$
Now applying the minimally labeled assumption to $\lambda(i)$ with
$k=j$ we get $\lambda(i)=d(\psi_{i})\cup\bigcup_{j'\in ne(i)-\{j\}}s_{ij'}.$
Thus,$\bigcup_{k\in\mathcal{T}_{i}^{-j}}\lambda(k)=d(\psi_{i})\cup\left(\bigcup_{j'\in ne(i)-\{j\}}s_{ij'}\cup\bigcup_{k'\in\mathcal{T}_{j'}^{-i}}\lambda(k')\right).$
By definition each separator $s_{ij'}\subseteq\lambda(j'),$ and hence
$s_{ij'}\subseteq\bigcup_{k'\in\mathcal{T}_{j'}^{-i}}\lambda(k'),$
so we can remove the $s_{ij'}$ from the previous expression, getting
$\bigcup_{k\in\mathcal{T}_{i}^{-j}}\lambda(k)=d(\psi_{i})\cup\left(\bigcup_{j'\in ne(i)-\{j\}}\bigcup_{k'\in\mathcal{T}_{j'}^{-i}}\lambda(k')\right).$
Now we can apply the induction hypothesis on each children $j$, getting
$\bigcup_{k'\in\mathcal{T}_{j'}^{-i}}\lambda(k')=d(\prod_{k'\in\mathcal{T}_{j'}^{-i}}\psi_{k'})$
and the proof is finished.\end{proof}
\begin{cor}
\label{lem:LambdasInDomain} Let $\phi=\phi_{1}\times\dots\times\phi_{n}$
be a valuation and let $\mathcal{\mathcal{T}}=(V,E,\lambda,U)$ a
minimally labeled covering join tree for this factorization. Then,
$d(\phi)=\bigcup_{i\in V}\lambda(i)$.\end{cor}
\begin{proof}
By induction on the height of the tree, parallel to the one of the
previous Lemma. \end{proof}
\begin{lem}
For any node $i$ of $V,$

\textup{
\begin{equation}
\lambda^{de}(i)=d(\prod_{j\in de(i)}\psi_{j})\label{eq:lambdaDIsProduct}
\end{equation}
}\end{lem}
\begin{proof}
Every descendant of $i$ lies on the subtree of one of its childs.
Thus, $\lambda^{de}(i)=\bigcup_{j\in ch(i)}\bigcup_{k\in\mathcal{T}_{j}^{-i}}\lambda(k)$
and by direct application of Lemma \ref{lem:Tree-Breakup} we get
$\lambda^{de}(i)=\bigcup_{j\in ch(i)}d(\prod_{k\in\mathcal{T}_{j}^{-i}}\psi_{k}))=d(\prod_{j\in ch(i)}\prod_{k\in\mathcal{T}_{j}^{-i}}\psi_{k}))=d(\prod_{j\in de(i)}\psi_{j}).$ \end{proof}
\begin{lem}
For any node $i$ of $V,$\textup{
\begin{equation}
\lambda^{nde}(i)=d(\prod_{j\in nde(i)}\psi_{j})\label{eq:lambdaNDIsProduct}
\end{equation}
}\end{lem}
\begin{proof}
Directly applying Lemma \ref{lem:Tree-Breakup} to the link $\{p_{i},j\},$
since the set of nodes in $\mathcal{T}_{p_{i}}^{-i}$ is exactly $nde(i).$ 
\end{proof}

\section{Piecewise and projective extensibility\label{app:Piecewise-and-projective}}

In this section we concentrate on proving proposition \ref{prop:completabilities}.
Let us start by recalling it. \completabilities*

We will provide an example of valuation algebras and extension system
in each of the four categories. 

A simple example of valuation algebra and extension system such that
none of the completabilities are satisfied is the one provided in
counterexample \ref{cnt:Boolean2}. As for the fourth category, any
valuation algebra induced by the semiring $(\mathbb{R},\max,\cdot)$
satisfies both piecewise and projective extensibility. The valuation
algebra presented in counterexample \ref{cnt:NotSquareOrdered} satisfies
piecewise completability but does not satisfy projective completability.
Next, we provide an example of valuation algebra satisfying projective
completability but not piecewise completability.

Let $U=\{x,y\}$ be a set with two variables. Let $D_{x}=D_{y}=\{0,1\}$
and $\Omega$ the set of all tuples. We have that $\langle U,\Omega\rangle$
are a variable system. Consider the valuation algebra induced by the
semiring $(\mathbb{R},\max,+)$. Let $\phi_{1}:\Omega_{X}\rightarrow\mathbb{R}$,
and $\phi_{2}:\Omega_{Y}\rightarrow\mathbb{R},$ be two valuations
defined as 
\[
\begin{array}{ccc}
\phi_{1}((x\mapsto0))=2 & \qquad & \phi_{2}((y\mapsto0))=2\\
\phi_{1}((x\mapsto1))=1 & \qquad & \phi_{2}((y\mapsto1))=1
\end{array}
\]
Taking $\Psi=\{\phi_{1}^{a}\times\phi_{2}^{b}\times(\phi_{1}^{\downarrow\emptyset})^{c}\times(\phi_{2}^{\downarrow\emptyset})^{d}\}$,
it is easy to prove that $(\Psi,U)$ fulfils the axioms of a valuation
algebra. 

Next, we have to define the extension sets in $(\Psi,U')$. We will
build a new extension system $\mathcal{\overline{W}}$ in the following
way:
\begin{itemize}
\item For $\phi_{1}$ we define its extensible solutions $\overline{W}_{\phi_{1}}^{\emptyset}(\diamond)=\overline{c}_{\phi_{1}}=\{(x\mapsto0),(x\mapsto1)\}$ 
\item For $\phi_{2}$ we define $\overline{W}_{\phi_{1}}^{\emptyset}(\diamond)=\overline{c}_{\phi_{2}}=\{(y\mapsto0),(y\mapsto1)\}.$
\item For any other valuation $\psi\in\Psi$, with $d(\psi)=X$ we define
$\overline{W}_{\psi}^{\emptyset}(\diamond)=\overline{c}_{\psi}=\{(x\mapsto0)\}.$
\item For any other valuation $\psi\in\Psi$, with $d(\psi)=Y$ we define
$\overline{W}_{\psi}^{\emptyset}(\diamond)=\overline{c}_{\psi}=\{(y\mapsto0)\}.$ 
\item For any other valuation $\psi\in\Psi$, with $d(\psi)=X\cup Y$ we
define $\overline{W}_{\psi}^{\emptyset}(\diamond)=\overline{c}_{\psi}=\{((x,y)\mapsto(0,0))\}.$ 
\end{itemize}
This definition guarantees that $\overline{\mathcal{W}}$ is an extension
system on $(\Psi,U).$ 

We will now see that the valuation algebra $(\Psi,U)$ with extension
system $\overline{\mathcal{W}}$ satisfies projectitve extensibility
but does not satisfy piecewise extensibility. 

For any valuation with domain $X$ it is immediate to prove that it
is projective extensible since there is no domain $\emptyset\subsetneq D\subsetneq X$.
Same holds for any valuation with domain $Y.$ 

For any valuation $\psi$, such that $d(\psi)=X\cup Y$ we have $\overline{c}_{\psi}=\{((x,y)\mapsto(0,0))\}.$
Additionaly it holds $\overline{c}_{\psi^{\downarrow X}}=\{(x\mapsto0)\}$
and $\overline{c}_{\psi^{\downarrow Y}}=\{(y\mapsto0)\}$. In particular
we have that$\psi$ is projective extenible.

We have just shown that all the valuations in $(\Psi,U)$ are projective
extensible. Hence we only have to find a valuation which is not piecewise
extensible. Let $\phi=\phi_{1}\times\phi_{2}$. Since $\overline{c}_{\phi^{\downarrow X}}=\{(x\mapsto0)\}$
and $\overline{c}_{\phi_{2}}=\{(y\mapsto0),(y\mapsto1)\}$ we have
\[
CO(c_{\phi^{\downarrow X}},\phi_{2})=\{(x,y)\rightarrow(0,0),(x,y)\rightarrow(0,1)\}\not\subseteq\bar{c}_{\phi}=\{((x,y)\mapsto(0,0))\}
\]

Hence $\phi$ is not piecewise extensible. In particular, all the
valuations in $(\Psi,U)$ with extension system $\overline{\mathcal{W}}$
are projective extensible but not all of them are projective extensible.
Indeed, it can be seen that the only piecewise extensible valuations
are $\phi_{1}$ and $\phi_{2}.$

\section{Some selective semirings properties}
\begin{defn}
Let $(R,+,\cdot)$ be a semiring. If for each $a\in R,$ $a+a=a,$
the semiring is \emph{idempotent.}\end{defn}
\begin{cor}
\label{cor:SelectiveSemiring}Let $(R,+,\cdot)$ be a commutative
semiring. Then $(R,+,\cdot)$ is selective if, and only if, $(R,+,\cdot)$
is totally ordered and idempotent.\end{cor}
\begin{proof}
Assume now that $(R,+,\cdot)$ is idempotent and totally ordered and
take $a,b\in R$. Without loss of generality we can assume $a\leq b$,
i.e. there is $c\in R$ such that $a+c=b$. Therefore $a+b=a+(a+c)=(a+a)+c=a+c=b.$
This proves the if part.

To prove the only if part, note that any selective semiring is idempotent.
Moreover, we have already seen that as a consequence of Proposition
3.4.7 in \cite{Gondran2008}, any selective semiring is totally ordered. \end{proof}
\begin{defn}
A selective semiring is \emph{strict monotonic} if whenever $c\neq0,$
$a<b$ implies that $a\cdot c<b\cdot c.$\emph{ }

A selective semiring is \emph{multiplicatively cancellative} if whenever
$c\neq0,$ $a\cdot c=b\cdot c$ if and only if \textbf{$a=b.$}\end{defn}
\begin{prop}
\label{prop:Strict-Iff-Cancellative}Let $(R,+,\cdot)$ be a selective
semiring. Then $(R,+,\cdot)$ is strict monotonic if and only if $(R,+,\cdot)$
is multiplicatively cancellative.\end{prop}
\begin{proof}
Assume that $(R,+,\cdot)$ is multiplicatively cancellative. Given
$a,b,c\in R$ with $c\neq0$ we want to see that $a<b\Rightarrow a\cdot c<b\cdot c$.
Since $a<b$ we have that $b=a+b$. By multiplying by $c$ at both
sides of the equality we get $b\cdot c=(a+b)\cdot c=a\cdot c+b\cdot c$.
Hence, there exist $d=b\cdot c\in R$ such that $a\cdot c+d=b\cdot c$.
By definition of the canonical order induced by $+$ we have $a\cdot c\leq b\cdot c.$
Since we have multiplicative cancellativity $a\cdot c=b\cdot c$ implies
$a=b$ which is a contractiction. Hence $a\cdot c\neq b\cdot c$.
In particular $a\cdot c<b\cdot c.$

Assume that $(R,+,\cdot)$ is strict monotonic. Given $a,b,c\in R$
with $a\cdot c\neq0$we want to see that $a=b\Leftrightarrow a\cdot c=b\cdot c$.
Notice that $a=b$ always implies $a\cdot c=b\cdot c,$so we only
have to prove the inverse implication. Assume $a\cdot c=b\cdot c$
holds. Since the semiring is totally ordered we have either $a\leq b$
or $b\leq a.$ Since $b\cdot c=a\cdot c\neq0$ we can assume without
loss of generality that $a\leq b.$ If $a\lneq b$ then by strict
monotonicity we have $a\cdot c\lneq b\cdot c$ which is a contradiction.
Hence $a=b.$
\end{proof}

\pagebreak

\bibliographystyle{plain}
\bibliography{library}

\end{document}